\documentclass{article} 
\usepackage{iclr2026,times}
\iclrfinalcopy
\usepackage{amsmath,amsfonts,bm}

\def\1{\bm{1}}

\def\vtheta{{\bm{\theta}}}

\def\mA{{\bm{A}}}

\def\mH{{\bm{H}}}

\def\mK{{\bm{K}}}

\def\mM{{\bm{M}}}

\def\mU{{\bm{U}}}

\def\mY{{\bm{Y}}}

\DeclareMathAlphabet{\mathsfit}{\encodingdefault}{\sfdefault}{m}{sl}
\SetMathAlphabet{\mathsfit}{bold}{\encodingdefault}{\sfdefault}{bx}{n}

\DeclareMathOperator*{\argmin}{arg\,min}

\DeclareMathOperator{\Tr}{Tr}

\newcommand{\vphi}{\bm{\phi}}
\newcommand{\valpha}{\bm{\alpha}}

\usepackage[hidelinks]{hyperref}       
\usepackage{url}            
\usepackage{booktabs}       
\usepackage{hhline}

\usepackage{multirow}
\usepackage{amsfonts}       
\usepackage{nicefrac}       
\usepackage{microtype}      
\usepackage{xcolor}         
\usepackage{enumitem}

\usepackage{amsmath}

\usepackage{amsmath}        
\usepackage{amssymb}
\usepackage{mathtools}
\usepackage{amsthm}
\usepackage{bbold}
\usepackage{yfonts}         
\usepackage{physics}        

\theoremstyle{plain}
\newtheorem{theorem}{Theorem}[section]
\newtheorem{axiom}[theorem]{Axiom}
\newtheorem{proposition}[theorem]{Proposition}
\newtheorem{lemma}[theorem]{Lemma}

\theoremstyle{definition}
\newtheorem{definition}[theorem]{Definition}

\theoremstyle{remark}

\usepackage{graphicx}
\usepackage{subfig}
\usepackage{soul}
\usepackage{algorithm}
\usepackage{algorithmic}

\usepackage{bm} 
\newcommand{\bif}[1]{\bm{\mathit{#1}}}
\newcommand{\itf}[1]{\mathit{#1}}

\title{Quantum Generator Kernels}

\author{Philipp Altmann
\thanks{Corresponding author. Contact: \texttt{philipp.altmann@ifi.lmu.de}}\\
LMU Munich, Germany
\And 
Maximilian Mansky\\
LMU Munich, Germany
\And 
Maximilian Zorn\\
LMU Munich, Germany
\And
Jonas Stein\\
LMU Munich, Germany
\And
Claudia Linnhoff-Popien\\
LMU Munich, Germany
}

\begin{document}

\maketitle

\begin{abstract}
Quantum kernel methods offer significant theoretical benefits by rendering classically inseparable features separable in quantum space. 
Yet, the practical application of \textit{Quantum Machine Learning} (QML), currently constrained by the limitations of \textit{Noisy Intermediate-Scale Quantum} (NISQ) hardware, necessitates effective strategies to compress and embed large-scale real-world data like images into the constrained capacities of existing quantum devices or simulators. 
To this end, we propose \textit{Quantum Generator Kernels} (QGKs), a generator-based approach to quantum kernels, comprising a set of \textit{Variational Generator Groups} (VGGs) that merge universal generators into a parameterizable operator, ensuring scalable coverage of the available quantum space. 
Thereby, we address shortcomings of current leading strategies employing hybrid architectures, which might prevent exploiting quantum computing's full potential due to fixed intermediate embedding processes. 
To optimize the kernel alignment to the target domain, we train a weight vector to parameterize the projection of the VGGs in the current data context. 
Our empirical results demonstrate superior projection and classification capabilities of the QGK compared to state-of-the-art quantum and classical kernel approaches and show its potential to serve as a versatile framework for various QML applications.
\end{abstract}

\section{Introduction} \label{sec:introduction}

Quantum computing offers fundamentally new paradigms for machine learning by exploiting quantum properties such as superposition and entanglement \citep{preskill2018quantum, biamonte2017quantum}. 
Among these, quantum kernel methods have shown promise for enhancing data separability via expressive feature maps that operate in high-dimensional Hilbert spaces, allowing them to capture structures that classical kernels cannot efficiently represent \citep{schuld2019quantum}. 
Despite these theoretical promises, the practical application of QML is currently limited by the capabilities of NISQ devices, which are still in their developmental stages \citep{preskill2018quantum}. 
Still limited in qubit capacities and subject to errors, embedding large-scale data into quantum devices is a significant hurdle that must be overcome to exploit the full potential of QML. 
Hybrid QML architectures bridge this gap, using classical pre-processing to embed data into quantum systems \citep{cerezo2021variational}. 
However, recent studies also highlight key limitations: many current approaches rely on fixed embeddings that do not scale well with high-dimensional inputs and are susceptible to barren plateaus during training \citep{mcclean2018barren}. 
Addressing this requires scalable, flexible, and learnable quantum embeddings that can exploit these properties while remaining parameter-efficient and robust to noisy hardware.

In this work, we propose a novel kernel architecture grounded in Lie algebraic generators, aggregated into parameterizable groups that project input data directly into quantum space. 
Our approach can be broken down into three steps: 
Construct a set of generators and merge them into \textit{Variational Generator Groups} (VGGs). 
We refer to a set of these groups as the \textit{Quantum Generator Kernel} (QGK), which is executed by its set of operators.
To improve alignment between data and the resulting kernel, we introduce a linear feature extractor that is pre-trained to project high-dimensional input into a compressed generator space.  
Unlike previous methods that rely on static gate-based embeddings, our QGK architecture employs Hamiltonian-driven unitaries with learnable generator weights, enabling expressive and scalable data encoding.  
By projecting high-dimensional data into a compact generator-weighted space, QGKs achieve high parameter efficiency per qubit and flexible embedding capacity, while effectively leveraging the expressive power of the full Hilbert space.
We validate the QGK both analytically and empirically: theoretical analyses confirm its expressivity and scalability, while experimental results demonstrate superior classification accuracy and robustness to noise across synthetic and real-world benchmarks.
We summarize our contributions as follows:
\begin{itemize}[left=8pt]
\item We introduce \textit{Variational Generator Groups} (VGGs), a novel embedding framework aggregating Lie-algebraic generators into parameterized Hermitian operators to enable expressive, data-dependent  state preparation and scalable alternatives to fixed gate-based encodings.
\item We propose the \textit{Quantum Generator Kernel} (QGK), a generator-driven kernel building on VGGs to use Hamiltonian evolution with data-conditioned weights, yielding compact, highly expressive feature maps with favorable parameter–qubit scaling.
\item We analyze the theoretical kernel properties of VGGs, 
characterizing entanglement capability, expressivity, parameter scalability, and computational complexity, demonstrating a classically efficient strategy for high-dimensional inputs that remains fully compatible with future fault-tolerant quantum execution.
\item We empirically evaluate the QGK on binary and multi-class benchmarks, including \texttt{MNIST} and \texttt{CIFAR10}, demonstrating consistent improvements over state-of-the-art classical and quantum kernels and robustness under realistic noise models.
\end{itemize}

\section{Background}

\paragraph{Kernel methods}
Kernel methods describe a map into a high-dimensional feature space $\psi$, used to project a $d$-dimensional data point $x\in\mathbb{R}^d$, into a space, where the distribution of data may then be suitable for linear separation. 
The kernel function $k(x_i, x_j) = \langle \varphi(x_i), \varphi(x_j)\rangle$ can be understood as the inner product in a (possibly high-dimensional) feature space induced by the feature map $\psi$.
As the feature map for all points may be computationally expensive, it can be important to reduce the number of calls to that function. The \emph{kernel trick} allows for calculating pairwise distances in the feature space without explicitly calculating the feature function. This is possible, for example, if the kernel is a positive definite map, independent of its dimension \citep{hofmann_kernel_2008}.
The pairwise distances $k(x_i, x_j)$ can be used in a \textit{support vector machine} (SVM) to solve a classification problem.
SVMs turn the classification problem into a quadratic optimization problem, defined on pairwise distances using the kernel trick:
\begin{equation}\label{eq:svm}
    \argmin_{\valpha} \frac12 \sum_{\begin{smallmatrix}i=1\\j=1\end{smallmatrix}}^n\valpha_i \valpha_j y_i y_j k(x_i, x_j) - \sum_{i=1}^n \valpha_i
\end{equation}
where $\valpha$ are the parameters of the support vector machine and $y$ are the labels of $n$ datapoints.
The quality of a feature map can be measured and adjusted to provide a better data embedding structure using kernel target alignment (KTA) \citep{cristianini_kernel-target_2001}.
KTA quantifies the similarity between a computed kernel matrix and an ideal target kernel, e.g., derived from the class labels.

\paragraph{Quantum computation}
Quantum systems are described by quantum states, represented as complex vectors encoding observables like spin or position. A single qubit can be written as $\ket{q} = \valpha \ket{0} + \beta \ket{1}$, where $\valpha, \beta \in \mathbb{C}$. Quantum operations on these states are unitary transformations $\mU$ such that $\ket{\Psi'} = \mU \ket{\Psi}$. In practice, these are decomposed into elementary one- and two-qubit gates and executed on quantum hardware. Measurement collapses the quantum state into classical outcomes \citet{nielsen_quantum_2010}.
Mathematically, such operations form the special unitary group $SU(N)$ for $N = 2^\eta$ qubits. As a Lie group, $SU(N)$ has a corresponding Lie algebra $\mathfrak{su}(N)$, a real vector space of skew-Hermitian, trace-zero matrices \citet{hall_lie_2013}.
The elements of the Lie algebra, known as generators, define the fundamental directions in which unitary quantum operations can be constructed.
Since the Lie algebra forms a linear vector space, multiple generators can be combined additively to form complex Hermitian operators. 
These operators are then mapped to unitary matrices via the exponential map, $\exp\colon \mathfrak{su}(N) \rightarrow SU(N)$.
In quantum computing, a natural basis for this algebra is given by traceless Pauli strings, i.e., tensor products of ${\sigma_x, \sigma_y, \sigma_z, \mathbb{1}}$.
These form the building blocks for quantum circuits and support gradient-based learning due to their manifold structure.

\paragraph{Quantum Kernel Methods}
A central difference in a system of $\eta$ qubits compared to classical bits is that the dimension of the mathematical space of $\eta$ qubits scales exponentially in $\eta$. 
Formally, the state of a system of qubits is a ray in a \emph{Hilbert space}, which is generally expressed by $\mathbb{C}^{2^\eta}$.
This exponential scaling motivates the quantum kernel method, in which the Hilbert space is used as the feature space, analogous to conventional kernel methods \citep{mengoni2019kernel}.
As shown in \citep{schuld2021quantum}, a large class of supervised quantum models are kernel methods. A deeper analysis of the mathematical structure of data embedding in quantum circuits shows that combining data uploading with parameterized quantum gates allows for arbitrary function approximation in the form of a Fourier series \citep{schuld2021effect}.

\section{Variational Generator Groups}\label{sec:vgg}

Information can be encoded into a quantum system either by initializing a quantum state with free parameters or by applying a parameterized operator to a fixed initial state. 
Given the limited qubit counts in current NISQ hardware, we aim to maximize parameter density per qubit.
A common quantum encoding approach, known as amplitude encoding, embeds classical data directly into a state vector within a $2^\eta$-dimensional Hilbert space, where each entry represents a complex amplitude \citep{biamonte2017quantum}. Accounting for both the real and imaginary components as well as the normalization constraint, this allows for a total of $2^{\eta+1} - 1$ free parameters.
An alternative and more hardware-friendly variant is rotational (angle) encoding, where classical values are mapped to the rotation angles of single-qubit gates (e.g., $R_x(\vtheta)$), applied to each qubit independently.
While this method is easier to implement and preserves data locality, it significantly limits the expressiveness of the encoding and only scales linearly with the number of qubits.
In contrast, encoding classical data via a unitary operator acting on a quantum state leverages the full expressive power of the Lie algebra $\mathfrak{su}(2^\eta)$, offering up to $2^{2\eta} - 1$ free parameters. 
This yields an exponential increase 
in representational capacity over state-based encodings, offering significantly greater flexibility and expressivity for quantum learning tasks. 

This insight motivates our approach: rather than embedding data directly into a quantum state, we build unitary operators from structured combinations of algebraic generators. 
Specifically, we construct a complete and well-behaved set of Hermitian generators that span a subalgebra of $\mathfrak{su}(2^\eta)$. 
The full derivation and construction procedure are detailed in \autoref{app:generator}, including \autoref{alg:generator_construction} outlining the algorithmic formulation.
Overall, this construction ensures linear independence, closure under commutation, and coverage of all valid operator directions in the Hilbert space, yielding a complete and hardware-compatible generator basis for building expressive and efficient quantum kernels:

\begin{theorem}\label{th:coverage}
Let $\mathfrak{H}$ be the set of Hermitian generators constructed from \autoref{alg:generator_construction}. Then $\mathfrak{H}$ spans a Lie subalgebra $\mathfrak{h} \subseteq \mathfrak{su}(2^\eta)$ that is closed under commutation, linearly independent, and expressible in terms of Pauli basis elements, ensuring algebraic validity and hardware implementability.
\end{theorem}
\begin{proof}
The generator set $\mathfrak{H}$ consists of three families of Hermitian matrices: off-diagonal real symmetric matrices (\autoref{eq: generators1}), off-diagonal imaginary anti-symmetric matrices (\autoref{eq: generators2}), and diagonal traceless real matrices (\autoref{eq: generators3}). 
These matrices span the entire space of traceless Hermitian operators on $\mathbb{C}^{2^\eta}$ and yield $4^\eta - 1$ linearly independent elements, matching the dimension of the Lie algebra $\mathfrak{su}(2^\eta)$. 
Their construction guarantees that the Lie algebra generated by $\mathfrak{H}$ is closed under commutation and coincides with $\mathfrak{su}(2^\eta)$.
Moreover, since the Pauli basis forms a complete operator basis of $\mathfrak{su}(2^\eta)$, each Hermitian generator $h_k \in \mathfrak{H}$ can be decomposed as a real linear combination of Pauli basis elements.
This establishes a direct correspondence between the generator-based formalism and standard quantum circuit constructions based on Pauli rotations.
\end{proof}
This implies that any unitary generated by $h_k$ can be synthesized into a gate-based quantum circuit using known decomposition techniques. 
Hence, $\mathfrak{H}$ not only provides complete algebraic coverage for encoding but also remains practically compatible with current quantum hardware.

With the set of generators $\mathfrak{H}$ in hand, it is now possible to encode a real-valued $g$-dimensional parameter vector $\vphi\in\mathbb{R}^g$ into a sequence of unitary operators.
Theoretically, for an 8-qubit system, we could encode $65,535$ parameters into this sequence. 
Therefore, we propose to merge them into \textit{Variational Generator Groups} (VGGs), combining multiple generators with a single parameter. 

To do so, we split the list of generators into equally sized partitions. 
Every group $\mathcal{G}\subset\mathfrak{H}$ constitutes a set of generators linearly combined to form a single Hermitian matrix $\hat{\mH}_i$, constructed according to \autoref{alg:vgg_construction}, which can be mapped to a unitary operator $\hat{\mU}$ using the time development of the quantum state, substituting the time dependence by a parameter element $\vphi_i$: 
\begin{align}
 \hat{\mU}_{\vphi_i} = e^{-i \cdot \vphi_i \cdot \hat{\mH}_i} \label{eq: paramintroduction} 
\end{align}

\begin{algorithm}
\caption{Construction of Variational Generator Groups (VGGs)}\label{alg:vgg_construction}
\begin{algorithmic}[1]
\REQUIRE Generator set $\mathfrak{H}$ from \autoref{alg:generator_construction}, number of groups $g$, projection width $w\in[0,\eta]$
\ENSURE Grouped generators $\hat{\mH}\gets \mathbf{0}_{(g,H,H)}$ \hfill{\color{gray}$\vartriangleright$ Initialize empty VGGs with $H\gets 2^\eta$.}
\STATE Construct ordering $idx$ by alternating generators from \eqref{eq: generators1}, \eqref{eq: generators2}, \eqref{eq: generators3} evenly.
\STATE $perm \gets \langle \lfloor k\cdot (w/\eta)\cdot g \rfloor \rangle_{k=0}^{|\mathfrak{H}|-1}$ \hfill{\color{gray}$\vartriangleright$ Target stride (width-controlled).}
\IF{$perm$ not a permutation}
\STATE $perm\gets \langle (k\cdot 2^{\lfloor \log_2((w/\eta)\cdot g)\rfloor})\bmod D\rangle_{k=0}^{|\mathfrak{H}|-1}$. \hfill{\color{gray}$\vartriangleright$ Ensure bijection.}
\ENDIF
\STATE $idx \gets idx[perm]$ \hfill{\color{gray}$\vartriangleright$ Apply stride permutation.}
\STATE{\textbf{for} $(i,j)\in[0,g[\;\times\;[0,|\mathfrak{H}|/g[$ \textbf{do}} \hfill{\color{gray}$\vartriangleright$ Merge generators into groups.}
\begin{ALC@g}\STATE $\hat{\mH}[i]\gets \hat{\mH}[i]+\mathfrak{H}[perm[i\mathfrak{H}|/g+j]]$ \hfill{\color{gray}$\vartriangleright$ Select permuted index and merge.}\end{ALC@g}
\STATE{\textbf{end for}}
\end{algorithmic}
\end{algorithm}

By grouping the generators, we are able to define the number of parameters that are introduced into the unitary. 
We choose the total number of groups $g$ (i.e., the number of embeddable parameters) as: 
\begin{align}\label{eq:groups}
&\quad\quad 
g = \frac{|\mathfrak{H}|}{\Gamma_\eta}
= \left(3\cdot2^\eta - 6\cdot\left(\eta\;\mathrm{mod}\;2\right) + 3\right)
\end{align}
to ensure scaling the number of generators per group $\Gamma_\eta$ (\autoref{eq:generators_per_group}) approximately exponentially with the total number of generators $|\mathfrak{H}|$.
The grouping process is further parameterized by the projection width $w$ that determines the stride at which generators are assigned to groups. 
By varying this projection width, we can control the density of generators per VGG (cf. \autoref{fig:projection comparison}). 
Consequently, \autoref{alg:vgg_construction} ensures an even distribution of generators across all groups.
Overall, using a wide stride helps promote linear independence among the grouped operators by mixing structurally diverse generators, which can improve parameter identifiability.
Further justification for this approach, including comparisons to ungrouped and fully grouped schemes, is discussed in \autoref{app:grouping}, while the empirical properties of the resulting VGGs are analyzed in \autoref{sec:analysis} and \autoref{app:grouping_hp}. 
The theoretical analysis in \autoref{app:grouping_analysis} furthermore shows that expressivity is preserved as long as the groups form a strict partition of $\mathfrak{H}$.

\begin{figure*}[t]\centering
  \includegraphics[width=0.8\linewidth]{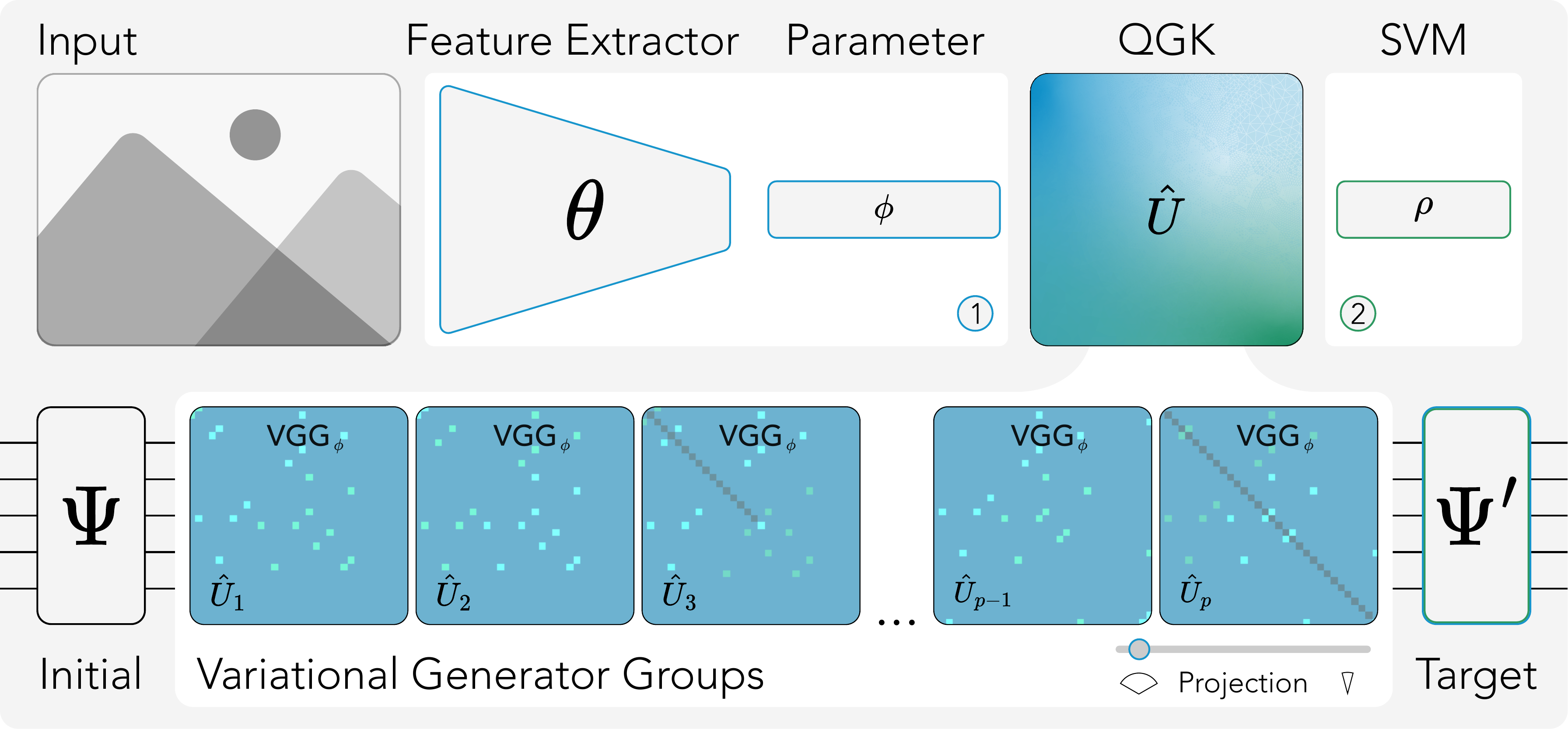}
  \caption{Quantum Generator Kernel: A generator-based quantum kernel architecture based on VGGs for parameterizable projection. Each colored matrix corresponds to one of $g=93$ \textit{Variational Generator Groups} (VGGs) merged for $\eta=5$ qubits, visualized as heatmaps of the magnitude (blue) and phase (green) of the resulting generators merged into the operator. The QGK is parameterized by the context $\vphi$, which is either given directly by the input or extracted from the input using a feature extractor, adapted during the pre-training phase (1) by updating the parameters $\vtheta$ to minimize the Kernel-Target Alignment (KTA) loss. 
  In phase (2), a support vector machine (SVM), parameterized by $\valpha$, is trained using the resulting QGK $\hat{\mU}$.
}\label{fig:QGK}
\end{figure*}

\section{Quantum Generator Kernels}\label{sec:qgk}

Via this procedure, we create a VGG with a unitary operator $\hat{\mU}_{\vphi_i}$ for every parameter element $\vphi_{i}$, which can be decomposed into quantum gates or applied directly onto an initial quantum state.
When composing the set of VGGs in a sequence, we obtain one condensed unitary that incorporates all parameters $\vphi$:
\begin{align}\hat{\mU}_{\vphi} = \exp\left(\sum_{i=1}^g -i\cdot\vphi_i\cdot\hat{\mH}_i\right) 
\label{eq:unitary} \end{align}
As shown in Eq.~\eqref{eq: closed_matmul_group}, the group of unitary matrices is closed regarding multiplication; hence, $\hat{\mU}_{\vphi}$ must also be unitary.
We refer to this unified operator $\hat{\mU}_{\vphi}$ as the \textit{Quantum Generator Kernel} (QGK), illustrated in \autoref{fig:QGK}, applied via 
\begin{align}\label{eq:qgk}\varphi(x) = \hat{\mU}_x \ket{\Psi} = \ket{\Psi'}\end{align}
to the initial uniform superposition state $\ket{\Psi}$, generated starting from the ground state $\ket{0}$ using $\otimes_{i=1}^\eta\mathcal{H}\ket{0} = \ket{\Psi},$ with $\eta$ qubits and the Hadamard gate $\mathcal{H}$ applied to all qubits. 

\newpage

To calculate the kernel matrix $\mK$ from our QGK, we use the fidelity as the distance between the states: 
\vspace{-1pt}\begin{align}\label{eq:kernel_matrix}
\mK = k(x_i,x_j) = \left| \bra{\Psi} \hat{\mU}^\dagger_{x_j} \hat{\mU}_{x_i}\ket{\Psi}\right|^2
\end{align}
To embed data using the QGK (cf. \autoref{fig:QGK}, phase 1), we can either use the input data to parameterize the VGGs directly, i.e., $\vphi=x\in \mathbb{R}^d$ as denoted in Eq.~\eqref{eq:qgk}, or use a feature extractor $\mathcal{F}_\vtheta: \mathbb{R}^d\mapsto\mathbb{R}^g$, parameterized by $\vtheta$, s.t. $\vphi = \mathcal{F}_\vtheta(x)$. 
Note that to directly embed input data, the number of groups $g$ must be set according to the input dimension $d$.
However, when using a feature extractor, we can decouple the input dimension from the number of generator groups. This enables dimensionality reduction, which we quantify via the compression factor $\gamma = d/g$.
With the embedded data, the kernel matrix can be used to fit a support vector machine (SVM) parameterized by $\valpha$ according to Eq.~\eqref{eq:svm}, to perform arbitrary classification tasks. 
To pre-train the parameterization $\vtheta=\langle W,b\rangle$ of the feature extractor, we suggest using the Kernel Target Alignment (KTA) loss, as suggested in \citep{hubregtsen2022training}:
\begin{align}\label{eq:KTA}
\mathcal{L}_{\text{KTA}} = 1 - \frac{\Tr(\mK \mY)}{\| \mK \|_{\mathrm{F}} \cdot \| \mY \|_{\mathrm{F}}}, 
\end{align}
with the kernel matrix $\mK$ and the classification targets $\mY$, where $\Tr(\mA)$ is the trace of matrix $\mA$ and $||\mA||_{\mathrm{F}}$ is the Frobenius norm.
In our experiments, we restrict $\mathcal{F}_\vtheta$ to a linear affine transformation of the form $\vphi = W x + b$.
This avoids introducing classical nonlinearities, and ensures that all expressive capacity stems from the quantum kernel.
The resulting model class remains confined to the RKHS induced by the kernel $\mathcal{K}_{\vphi}$ (see \autoref{app:linear_projection} for a theoretical analysis).

\section{Related Work}

\paragraph{Embedding Processes}
To encode classical data into quantum systems, various strategies have been developed, ranging from fixed mappings such as \textit{basis encoding}, where binary values $b \in \{0,1\}$ are mapped to computational basis states $\ket{b}$, to more compact methods such as \textit{amplitude encoding}, where a normalized vector $v \in \mathbb{R}^{2^\eta}$ is embedded directly into the amplitudes of a quantum state $\ket{\Psi}$. While amplitude encoding uses only $\eta$ qubits to represent exponentially large vectors, its practical use is limited due to costly state preparation and reduced kernel expressivity unless followed by complex unitaries~\citep{10.1109/TCAD.2023.3244885, schuld2021effect, schuld2019quantum, huang2021power}. \textit{Angle encodings}, or Pauli rotational embeddings, instead map real-valued inputs into single-axis rotations such as $R_X(x)$ or $R_Z(x)$, and form the basis of many variational quantum circuits. Their expressivity is often enhanced by \textit{data reuploading}~\citep{perez2020data}, which introduces nonlinearity through repeated injection of inputs, enabling the circuit to approximate Fourier-like transformations~\citep{PhysRevA.109.042421}. The Quantum Embedding Kernel (QEK)~\citep{hubregtsen2022training} leverages this mechanism to train kernel functions via KTA maximization. However, such methods rely on axis-aligned encodings and fixed circuit structures. 
In contrast, recent work has explored more expressive multi-axis or multi-qubit rotational embeddings.
In this context, our \textit{Variational Generator Groups} (VGGs) can be seen as a structured and systematic \emph{generalization} of multi-qubit angle encoding. 
Instead of encoding each input feature through a single-axis rotation, VGGs construct
unitaries from grouped, algebraically structured combinations of Hermitian generators $\hat{\mH}_k = \sum_{h \in \mathcal{G}_k \subset \mathfrak{H}} h$, where each $h$ is a linear combination of Pauli strings spanning the full Pauli basis (cf.~\autoref{th:coverage}). 
The resulting unitary transformations $\text{exp} (-i \vphi_k \hat{\mH}_k)$ 
implement data-dependent, \emph{multi-axis} and \emph{multi-qubit} transformations within the Lie algebra $\mathfrak{su}(2^\eta)$.
This distinguishes VGGs from traditional angle encoding in two key ways:
While angle encodings act on one Pauli axis per qubit, VGGs use grouped generators that cover structured subspaces of $\mathfrak{su}(2^\eta)$, enabling richer transformations per parameter.
Each group induces entangling, correlated rotations across qubits, going beyond independent, axis‑aligned operations.
Thus, VGGs form a principled extension of multi‑axis angle encoding, bridging fixed rotational embeddings and full Hamiltonian variational models
This allows for greater expressivity per qubit, making it particularly effective for kernel learning in qubit-constrained settings.

\paragraph{Hybrid QML}
Motivated by the limited capabilities of current quantum hardware, hybrid quantum-classical machine learning approaches have been explored in literature, which add pre- and postprocessing layers to the quantum model \citep{Mari2020transferlearningin}. While this approach allows the exploration of problems that are beyond the capabilities of current quantum hardware by handling larger data sizes, it blurs the individual contributions of the classical and the quantum components to the overall solution quality \citep{10.5220/0011772400003393,10.5220/0012381600003636}. 
This issue is particularly pronounced in \textit{Dressed Quantum Circuits} (DQC) \citep{Mari2020transferlearningin}, where nonlinear neural networks appear both before and after the quantum circuit. By the universal approximation theorem \citep{hornik1989multilayer}, these classical networks can approximate arbitrary continuous functions even if the quantum circuit acts trivially, making it difficult to attribute improvements to quantum processing. In contrast, our QGK uses only an affine transformation to parameterize generator weights and no nonlinear post‑processing, ensuring that expressivity is strictly governed by the quantum kernel (see \autoref{app:linear_projection}).
A recent example combining the above principles is the \textit{Hardware Efficient Embedding} (HEE) \citep{thanasilp2024exponential}, where input-dependent rotations are arranged in layered circuits interleaved with entangling gates (e.g., CNOT or CZ). These embeddings are expressive and compatible with near-term devices, but they can lead to exponentially concentrated kernel values and barren features unless carefully controlled.
To adapt the input dimensionality to the limited number of qubits, the authors apply \textit{principal component analysis} (PCA).
Even though the input dimensionality of our proposed generator-based approach scales exponentially with the number of qubits, in contrast to the linear scaling of the HEE, its preprocessing denotes a similar approach to the linear layer we introduce.
However, rather than static feature extraction, we further use this pre-processing to pre-train the projection of the kernel. 
Another recent direction is the \textit{Projected Quantum Kernel} (PQK) \citep{huang2021power}, which avoids fidelity-based kernels by extracting quantum features through the \textit{one‑particle reduced density matrix} (1‑RDM) and subsequently applying a classical kernel function.
Similar to HEE, the PQK is limited by linear input scaling and relies heavily on classical compression techniques such as PCA. 
Furthermore, it may suffer when the underlying label structure is not well aligned with the kernel geometry induced by the projected quantum features.
In contrast, the QGK supports high-dimensional inputs via grouped Hamiltonian encoding and offers kernel-target alignment (KTA) pre-training to adapt the kernel to task-specific structures.

\paragraph{Generator-based approaches}
The exploration of generator-based quantum computing is rather recent, compared to the longer exploration of gate-based variational circuits \citep{nielsen_quantum_2010}. Generators are used in more mathematical explorations of quantum computing \citep{mansky_near-optimal_2023}. In particular, they are found in the treatment of barren plateaus \citep{arrasmith_effect_2021, goh_lie-algebraic_2023, ragone_unified_2024} and specialized circuits that focus on restricting subspaces \citep{schatzki_theoretical_2024, nguyen_theory_2024}.
A direct equivalent to the generators does not exist in classical machine learning, owing to the different mathematical structure \citep{bronstein_geometric_2021}.

\section{Empirical Analysis}\label{sec:analysis}

In this section, we aim to provide an empirical analysis of the resulting properties of the proposed \textit{Quantum Generator Kernel} (QGK) and compare them to the \textit{Quantum Embedding Kernel} (QEK), \textit{Hardware Efficient Embedding} (HEE), and \textit{Projected Quantum Kernel} (PQK) as representative state-of-the-art approaches to quantum kernel methods, as well as classical \textit{Radial Basis Function} (RBF) and Linear kernels.
\autoref{fig:scalability-analysis} summarizes our results. 

\begin{figure*}[ht] \centering
\subfloat[Parameters]{\includegraphics[width=0.2\linewidth]{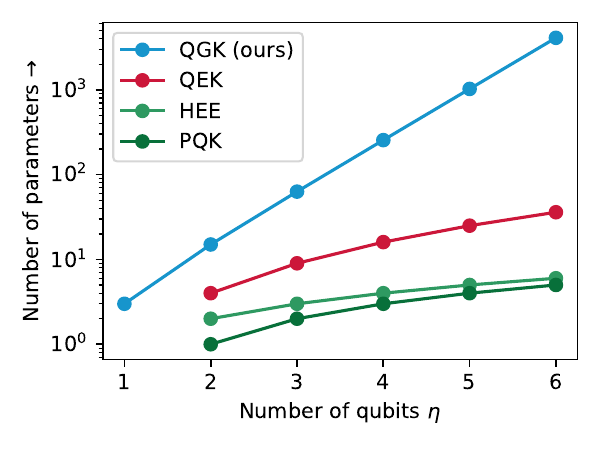}\label{fig:scale:params}}
\subfloat[Entanglement]{\includegraphics[width=0.2\linewidth]{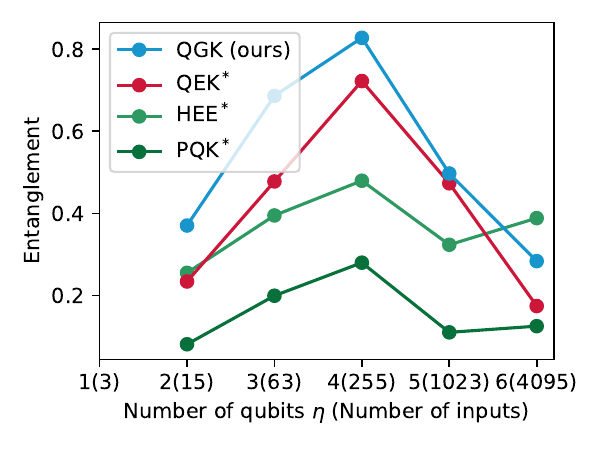}\label{fig:scale:entangle}}
\subfloat[Expressibility]{\includegraphics[width=0.2\linewidth]{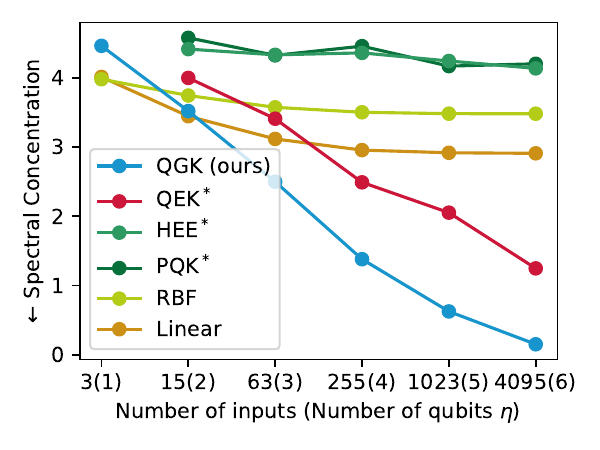}\label{fig:scale:expressibility}}
\subfloat[Complexity]{\includegraphics[width=0.2\linewidth]{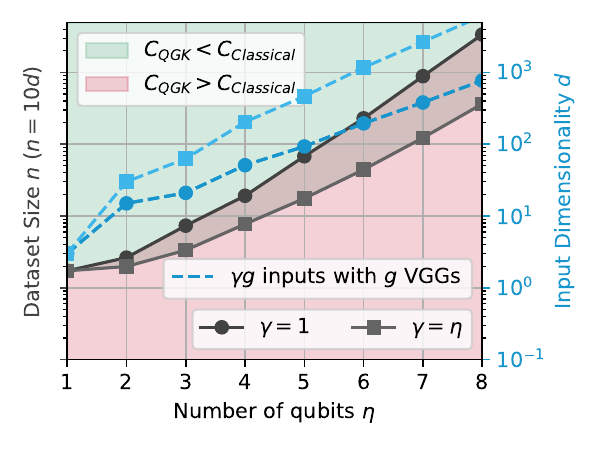}\label{fig:scale:complexity}}

\caption{Comparing the properties of the QGK to classical and quantum kernels w.r.t. the number of available qubits $\eta$, regarding the: \protect\subref{fig:scale:params} number of available parameters, \protect\subref{fig:scale:entangle} entanglement capability by means of the Meyer-Wallach measure, \protect\subref{fig:scale:expressibility} expressibility by means of the spectral concentration, and \protect\subref{fig:scale:complexity} computational complexity, showing the complexity breakeven dataset size $n$ when classically simulating the QGK (left y-axis) and the number of VGGs and resulting inputs (right y-axis, blue), scaled to $n=10d$ ($n\gg d$). The properties of the QEK$^*$, HEE$^*$, and PQK$^*$ are reported w.r.t the input scalability shown in \protect\subref{fig:scale:params}\protect\footnotemark. Refer to \autoref{app:grouping_hp} for an extended analsyis.
}
\label{fig:scalability-analysis}
\end{figure*}

\paragraph{Parameter efficiency} \autoref{fig:scale:params} shows the scalability of the number of parameters w.r.t. the number of qubits. 
For the QEK, we assume a maximum of $\eta^2$ input parameters, i.e., using $\eta$ layers to embed the input data, scaling quadratically with the number of qubits.
Utilizing generators in Hilbert space rather than fixed gate-based embedding processes, our approach shows significantly improved qubit efficiency, scaling exponentially with the number of qubits.
Thus, QGK offers an overall higher data capacity, which is especially crucial for realizing real-world applications on near-term quantum devices.
To ensure a fair comparison with equal parameter counts, the following analysis is based on the maximum number of QGK parameters, where the QEK is extended by parameterized reuploading layers to replenish the additional parameters.  
As mentioned above, the input capacity of methods such as HEE and PQK scales only linearly with the number of qubits.

\paragraph{Kernel properties} 
\autoref{fig:scale:entangle} shows the maximum entanglement capabilities by means of the Meyer-Wallach measure.
For all approaches, the entanglement capability peaks around four qubits and steadily decreases afterwards. 
Due to the mutli-qubit nature of the grouped generators, the QGK shows the overall highest entanglement capabilites. 
\autoref{fig:scale:expressibility} compares the kernels' spectral concentration as a measure of expressibility. 
Notably, the QGK exhibits the highest expressibility (low spectral concentration), indicating strong resilience against expressibility collapse and exponential concentration, despite operating in exponentially larger Hilbert spaces.
Also, benefiting from its generator-based structure with high parameter density and even coverage of the available parameter space, the QGK shows increasing expressibiltiy with increasing qubit count.
The compared kernels show less-decreasing concentration with increasing input size, suggesting weaker scalability with added parameters. 
These results highlight QGK’s favorable expressibility-to-learnability trade-off, making it a scalable and efficient alternative for classical and quantum kernels.
To assess architectural sensitivity, we analyze the impact of group size and projection stride in \autoref{app:grouping_hp}. 
The results confirm that expressibility and entanglement remain stable across all tested configurations, with exponential group scaling and wide stride yielding the most consistent and expressive feature maps. Finer groupings enhance expressiveness slightly, while projection width has only marginal impact. These findings reinforce the theoretical robustness guarantees presented in \autoref{app:grouping_analysis} and support our default grouping design as a scalable and well-conditioned choice.

\footnotetext[1]{Matching the input-scale of our QGK would require tens to thousands of qubits, which is infeasible for simulation or current hardware without aggressive classical dimensionality reduction that would obscure the intrinsic quantum embedding evaluated.}

Despite the exponential scaling in qubit number $\eta$, the QGK remains classically simulable due to its decomposition into tensor-efficient operations: generator construction, input projection, quantum evolution, and pairwise kernel evaluation. 
Unlike classical kernels such as RBF or Linear, which incur $\mathcal{O}(n^2 \cdot d)$ cost for computing the similarity matrix, QGK scales as $\mathcal{O}(4^\eta + n \cdot \gamma \cdot g^2 + n \cdot 8^\eta + n^2 \cdot 2^\eta)$, with $g$ VGGs and the compression ratio $\gamma=d/g$. 
This structure favors sample-efficient scenarios revealed by the complexity analysis shown in \autoref{fig:scale:complexity}, indicating two key properties of the QGK: 
(i) for low qubit number ($\eta\leq5$) QGK offers a computational advantage over classical kernels when using a 1:1 mapping ($\gamma=1$).
(ii) to enable larger-scale applications (e.g., $d > 100$) efficiently, hybrid approaches with $\gamma > 1$ are required to compress the input. 
E.g., using $\gamma=\eta$ pushes the lower efficiency bound further down, such that the input dimension reference ($d=\eta g$ for $g$ VGGs, light blue) does not intersect anymore. 
This hybrid execution enables scalable and classically feasible QGK training in the NISQ era and lays the groundwork for full quantum execution on future fault-tolerant architectures. A detailed analysis, including complexity thresholds, approximations, and formal bounds, is provided in \autoref{app:complexity}.
Importantly, while the QGK provides a path to handling large input dimensionality efficiently it shares the same quadratic complexity in the number of training samples $n$ as all classical kernel methods due to the $\mathcal{O}(n^2)$ kernel matrix computation. As with classical approaches, this can become prohibitive for very large datasets.
To mitigate this, kernel approximation techniques such as the Nyström method \citep{williams2000using} or random feature expansions \citep{rahimi2007random} can be used to reduce training complexity to near-linear in $n$, with minimal performance degradation.

\section{Evaluation}\label{sec:eval}

To empirically validate the properties and advantages discussed above, this section demonstrates the kernels' trainability, scalability, and hardware applicability across various tasks.
As small-scale synthetic benchmarks, we use \texttt{moons} and \texttt{circle} with $n\!=\!200$ and $d\!=\!2$ input features, augmented with 20\% noise \citep{pedregosa2011scikit}. 
Furthermore, we use the \texttt{bank} dataset \citep{moro2014data,bank_marketing_222} ($n\!=\!2000$, $d\!=\!16$, 2 classes) as well as the 10-class \texttt{MNIST} \citep{lecun1998gradient} ($d\!=\!784$) and \texttt{CIFAR10} citep{Krizhevsky09learningmultiple} ($d\!=\!3072$) datasets with $n\!=\!10000$, to show scalable real-world applicability.
In addition to the QEK~\citep{hubregtsen2022training}, PQK~\citep{huang2021power}, and HEE~\citep{thanasilp2024exponential}, we evaluate \textit{Radial Basis Functions} (RBF), \textit{Linear Kernels}, and a small \textit{Multi-Layer Perceptron} (MLP) as classical state-of-the-art baselines. 
To ensure even capabilities, QEK and HEE circuits comprise one data-reuploading layer, i.e., $2\cdot d/\eta$ layers parameterized by Y- and CZ-rotations for the QEK, and $2$ X-rotational input embedding layers for the HEE.
For the PQK baseline, we follow \citet{huang2021power}, projecting to $\eta-1$ input features via PCA and computing an RBF kernel over the one-particle reduced density matrix (1-RDM). Relabeling is omitted to reflect general-purpose scenarios.
To parameterize the feature projection of the QGK, we use a single linear layer. 
The baseline MLP is using a single hidden layer of size $g$, matching the number of generator groups.
To provide an ablation quantifying the impact of using KTA to train the kernel, we also report the untrained QGK performance (\textit{QGK Static}).
To additionally quantify the impact of the generator-based kernel itself and delimit it from the classical preprocessing, we adapted our linear pre-processing to the classical Linear Kernel (\textit{Linear KTA}) and HEE (\textit{HEE Linear}), whereas \textit{HEE} refers to the untrained approach using PCA for feature extraction.  
We pre-train the embedding parameterization and baseline MLP for $100$ epochs using the Adam optimizer with learning rate $10^{-(\eta-1)}$.
Unless stated otherwise, we use $\eta = 2$ for the binary and $\eta = 5$ for the multi-class benchmarks.
As a general performance metric, we use the classification accuracy on an unseen $10\%$ test-split, additionally reporting the \textit{Kernel Target Alignment} (KTA).
All results are averaged over eight random seeds with 95\% confidence intervals.
Non-pretrained approaches are shown as dashed horizontal lines. 
All evaluations were conducted on Apple M2 Ultra hardware with 192GB memory and a total compute time of approximately 96 hours using torchquantum \citep{hanruiwang2022quantumnas}.

\begin{figure*}[t]\centering
\includegraphics[width=0.95\linewidth]{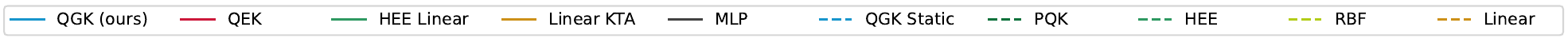}\\ \vspace{-1em}
  \subfloat[\centering\texttt{moons}\label{fig:training:moons}]{\includegraphics[width=0.2\linewidth]{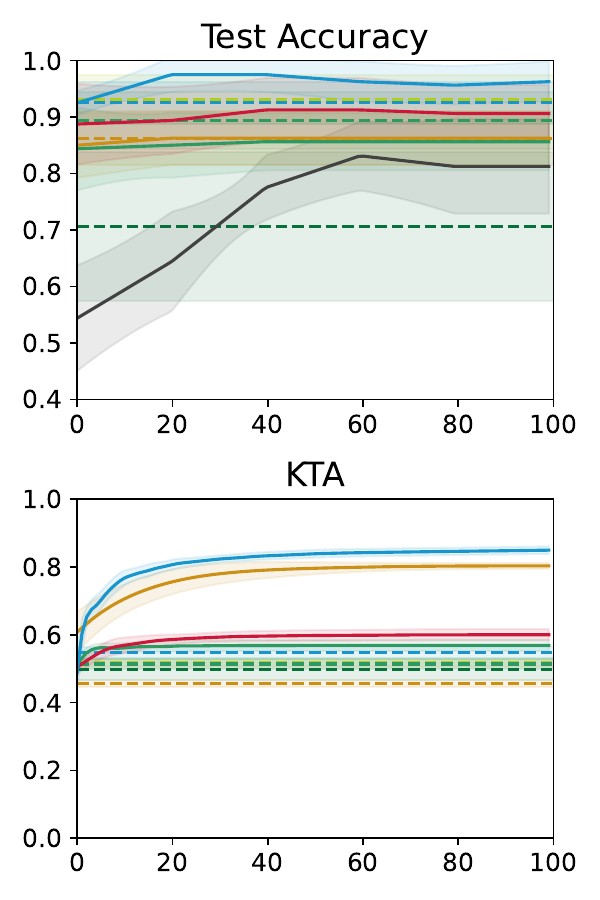}} 
  \subfloat[\centering\texttt{circles}\label{fig:training:circles}]{\includegraphics[width=0.2\linewidth]{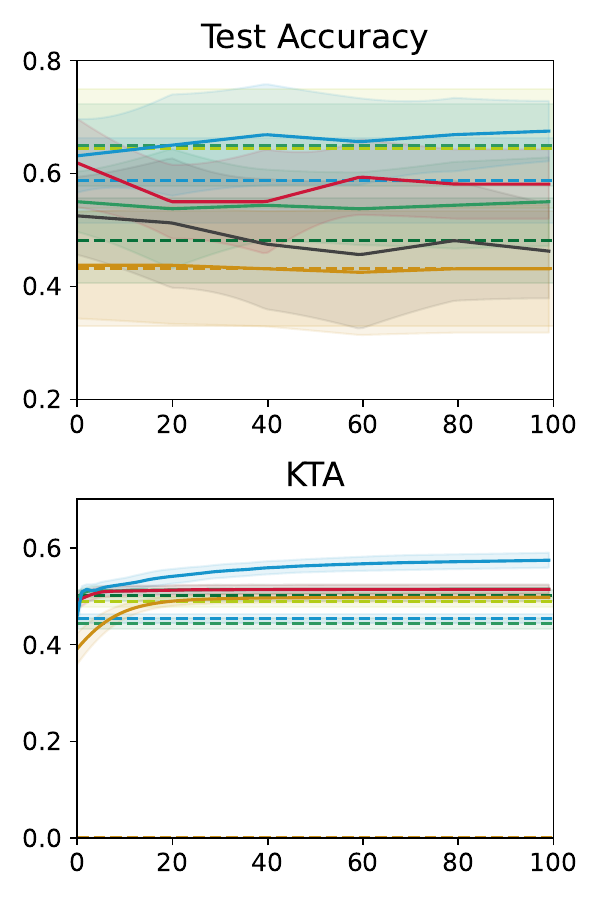}} 
  \subfloat[\centering\texttt{bank}\label{fig:training:bank}]{\includegraphics[width=0.2\linewidth]{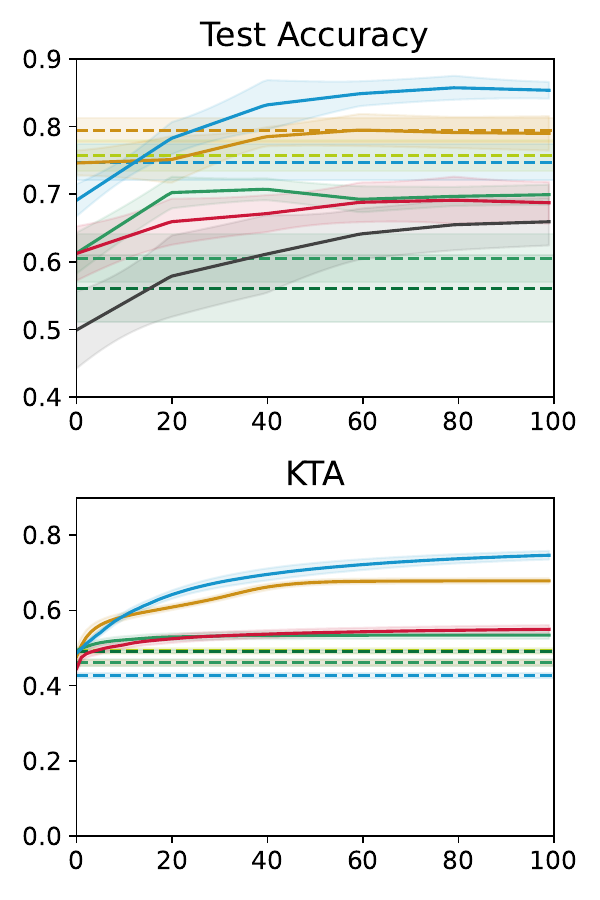}}
  \subfloat[\centering\texttt{MNIST}\label{fig:training:mnist}]{\includegraphics[width=0.2\linewidth]{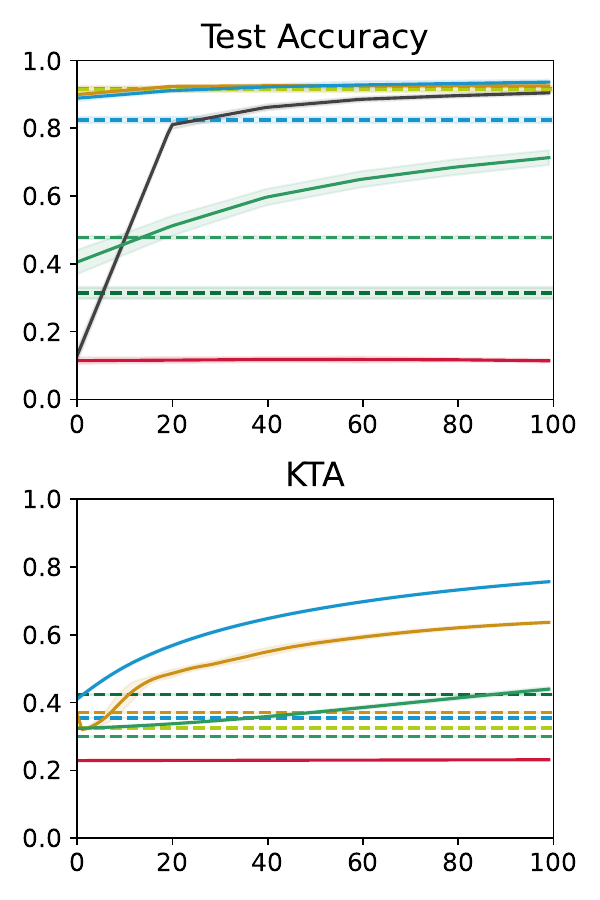}} 
  \subfloat[\centering\texttt{CIFAR10}\label{fig:training:cifar}]{\includegraphics[width=0.2\linewidth]{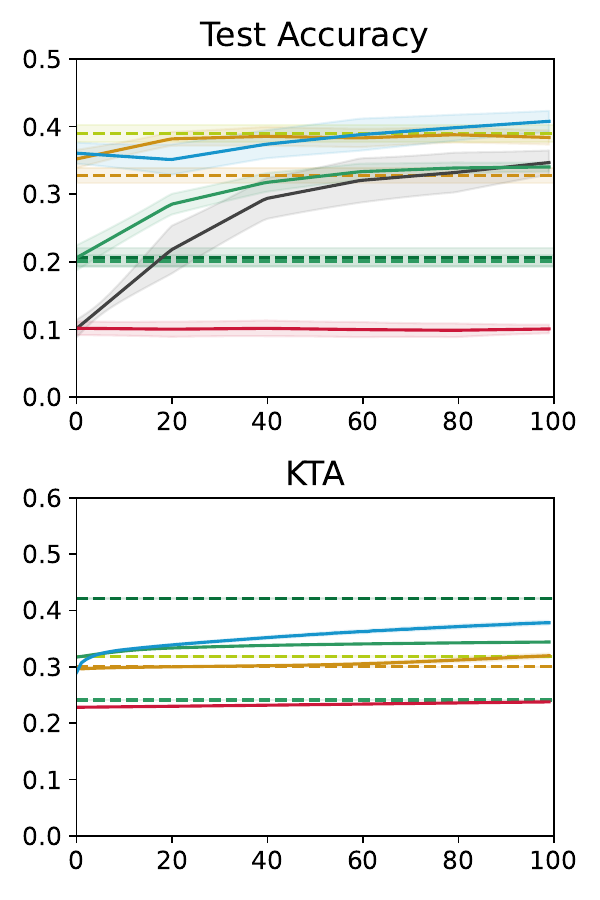}} 
  \caption{Training performance of the QGK(blue), QEK(red), PQK(dark green), HEE(green), RBF(yellow), and Linear(orange) Kernels and MLP(grey) w.r.t. the Test Accuracy (top) and KTA (bottom) in the \texttt{moons}, \texttt{circles}, \texttt{bank}, \texttt{MNIST}, and \texttt{CIFAR10} benchmarks, with the QGK outperforming all compared quantum kernels and classical kernels baselines. Final accuracies are summarized in \autoref{tab:final_acc}. }\label{fig:training}
\end{figure*}

\paragraph{Small-scale tasks}
The training results for the \texttt{moons} dataset are shown in \autoref{fig:training:moons}. 
Notably, QGK outperforms both embedding-based approaches.
Training the projection increases the final test accuracy from $93\%$ to $96\%$, and even outperforms both classical approaches, with a final KTA above $0.8$.
Similar performance trends are prevalent in the \texttt{circles} benchmark shown in \autoref{fig:training:circles}, where the QGK outperforms all compared quantum and classical approaches. 
The overall higher accuracy compared to the HEE Linear and KTA Linear ablations further underscores the effectiveness of our generator-based kernel over embedding-based alternatives and beyond the linear preprocessing.
Looking at the results for the real-world \texttt{bank} dataset in \autoref{fig:training:bank}, the QGK continues to expand its performance lead. 
Despite the increased complexity introduced by the 16-dimensional input space, it achieves a mean final test accuracy of $86\%$, significantly outperforming all competing methods. 
This highlights QGK’s ability to maintain robustness and accuracy even in higher-dimensional, non-trivial learning tasks.

\paragraph{Large-Scale Tasks}
To evaluate the scalability of the QGK, we investigate its performance on the 10-class \texttt{MNIST} and \texttt{CIFAR10} benchmarks, shown in \autoref{fig:training:mnist} and \autoref{fig:training:cifar}. On \texttt{MNIST}, QGK achieves a final test accuracy of around $94\%$, surpassing the performance of all classical baselines, and significantly outperforming all compared quantum kernels. 
Even without any trained projection, QGK Static achieves strong performance, highlighting the expressive nature of the generator-based representation.
On \texttt{CIFAR10}, a considerably more challenging dataset ($d\!=\!3072,\ n\!=\!10000$), QGK still achieves the highest accuracy among all tested methods around $41\%$, despite the strong compression required for five-qubit compatibility. In contrast, all quantum baselines (QEK, HEE, PQK) and classical approaches (RBF, Linear, MLP) show limited performance under the same constraints. These results illustrate the strong adaptability of QGK under large input dimensionality and its ability to generalize beyond synthetic or low-dimensional tasks.

\paragraph{Hardware Compatibility}
To assess compatibility with current quantum hardware and evaluate robustness to noise, we compiled all compared circuits to IBM’s Falcon architecture and simulated them using realistic noise models\footnote{Using the 27 qubit IBM Falcon processor~\textit{FakeToronto}.} using datasets with $n\!=\!200$. 
As summarized in \autoref{tab:result_summary}, for the binary classification tasks, all methods maintain manageable depths below 100, allowing realistic simulation under noise. 
Notably, QGK outperforms all compared approaches even when executed under realistic hardware noise. 
On larger-scale tasks, the compiled depths diverge, reflecting significant differences in the encoding strategies (cf. \autoref{tab:hyperparameters}).
HEE achieves the lowest depth but only encodes five features (less than 1\% of the input dimension, even for the lower-dimensional \texttt{MNIST}), relying heavily on classical preprocessing, arguably limiting its comparability. 
At the other extreme, QEK encodes all features without classical reduction, but at the cost of unmanageable circuit depths exceeding 10k, demonstrating its poor scalability due to the fixed embedding structure. 
In contrast, QGK strikes a practical balance: grouping generators to embed 93 features using five qubits yields manageable depths below 5k, without heavy preprocessing. 
The simulated QGK results of $81\%$ and $26\%$ for \texttt{MNIST} and \texttt{CIFAR10} respectively, still significantly outperform all compared quantum kernels and only slightly underperform the noise-free classical execution.
Notably, QGK’s structured design allows further depth reductions via generator pruning, paving the way for efficient deployment on fault-tolerant quantum systems.

\begin{table*}[ht]\centering\tiny
\resizebox{\textwidth}{!}{%
\begin{tabular}{|l@{\hspace{2pt}}c|l@{}c|l@{}c|l@{}c|l@{}c|l@{}c|l@{}c|}\hline
Method  &
& \multicolumn{2}{|c|}{\texttt{moons} ($d{=}2,\eta{=}2$)} 
& \multicolumn{2}{|c|}{\texttt{circles} ($d{=}2,\eta{=}2$)} 
& \multicolumn{2}{|c|}{\texttt{bank} ($d{=}16,\eta{=}2$)} 
& \multicolumn{2}{|c|}{\texttt{MNIST} ($d{=}784,\eta{=}5$)} 
& \multicolumn{2}{|c|}{\texttt{CIFAR10} ($d{=}3072,\eta{=}5$)} \\\hline

QGK (ours) & \multirow{4}{*}{\rotatebox[origin=c]{90}{noisy}}
& \multirow{4}{*}{\rotatebox[origin=c]{90}{$n{=}200$}}
& $\bif{96\% (28)}$ & \multirow{4}{*}{\rotatebox[origin=c]{90}{$n{=}200$}}
& $\bif{67\% (28)}$ & \multirow{4}{*}{\rotatebox[origin=c]{90}{$n{=}200$}}
& $\bif{88\% (28)}$ & \multirow{4}{*}{\rotatebox[origin=c]{90}{$n{=}200$}}
& $\bif{81\% (4754)}$ & \multirow{4}{*}{\rotatebox[origin=c]{90}{$n{=}200$}}
& $\bif{26\% (4754)}$ \\

QEK &
& & $\itf{86\%\ (29)}$
& & $\itf{59\%\ (29)}$
& & $\itf{66\%\ (191)}$
& & $11\%\ (12006)$
& & $11\%\ (47223)$ \\

HEE &
& & $\itf{88\%\ (18)}$
& & $\itf{61\%\ (18)}$
& & $\itf{59\%\ (18)}$
& & $21\%\ (53)$
& & $17\%\ (53)$ \\

PQK &
& & $\itf{53\%\ (362)}$
& & $\itf{48\%\ (362)}$
& & $\itf{51\%\ (368)}$
& & $19\%\ (416)$
& & $14\%\ (416)$ \\

\hline

QGK (ours) & \multirow{4}{*}{\rotatebox[origin=c]{90}{exact}}
& \multirow{4}{*}{\rotatebox[origin=c]{90}{$n{=}200$}}
& $\mathbf{0.96 \pm 0.04}$ & \multirow{4}{*}{\rotatebox[origin=c]{90}{$n{=}200$}}
& $\mathbf{0.68 \pm 0.05}$ & \multirow{4}{*}{\rotatebox[origin=c]{90}{$n{=}2k$}}
& $\mathbf{0.85 \pm 0.01}$ & \multirow{4}{*}{\rotatebox[origin=c]{90}{$n{=}10k$}}
& $\mathbf{0.94 \pm 0.01}$ & \multirow{4}{*}{\rotatebox[origin=c]{90}{$n{=}10k$}}
& $\mathbf{0.41 \pm 0.02}$ \\

MLP &
& & $0.81 \pm 0.08$
& & $0.46 \pm 0.08$
& & $0.66 \pm 0.04$
& & $0.90 \pm 0.01$
& & $0.35 \pm 0.02$ \\

RBF &
& & $0.93 \pm 0.04$
& & $0.64 \pm 0.11$
& & $0.76 \pm 0.02$
& & $0.92 \pm 0.01$
& & $0.39 \pm 0.01$ \\

Linear &
& & $0.86 \pm 0.05$
& & $0.43 \pm 0.10$
& & $0.79 \pm 0.02$
& & $0.92 \pm 0.01$
& & $0.33 \pm 0.01$ \\
\hline
\end{tabular}}

\caption{Final test accuracies across five benchmarks, comparing QGK (ours) with quantum kernels (QEK, HEE, PQK) and classical (MLP, RBF, Linear) baselines. Upper: results obtained from noisy circuit simulation on IBM Falcon hardware \citep{qiskit2024}, with the compiled circuit depth in parentheses. Lower: Final accuracies under noise-free execution as reported in \autoref{fig:training}. See Tabs.~\ref{tab:simulation_acc},~\ref{tab:final_acc} for extended results. Even under hardware noise, QGK consistently achieves the highest accuracy, demonstrating superior robustness and scalability across synthetic and real-world datasets.
}\label{tab:result_summary}
\end{table*}

\section{Conclusion}\label{sec:conclusion}

In this paper, we introduced \textit{Quantum Generator Kernels} (QGKs), a novel generator-based approach to quantum kernel methods. 
The QGK consists of \textit{Variational Generator Groups} (VGGs) that merge a set of universal generators into parameterizable groups.
Building upon universal generators in Hilbert space, QGKs offer significantly improved parameter scalability compared to common gate-based approaches, employing fixed embedding processes. 
Empirical studies across five benchmarks demonstrate that the QGK achieves superior trainability and classification accuracy, consistently outperforming quantum kernels and  classical baselines, even under hardware noise.

\paragraph{Key Results}
On both synthetic binary tasks, QGK outperforms classical and quantum embedding-based methods, showcasing strong expressiveness under limited quantum resources. 
On the real-world \texttt{bank} dataset, QGK maintains a clear lead, reaching $85\%$ accuracy despite the 16-dimensional input space. 
On the larger-scale \texttt{MNIST} benchmark, QGK reaches $94\%$ accuracy, matching the best classical kernel (Linear) and significantly outperforming all quantum baselines.
On the more complex \texttt{CIFAR10}, QGK achieves $41\%$, clearly surpassing QEK ($10\%$), HEE ($20\%$), and even Linear kernels ($33\%$) and small MLPs ($35\%$). 
These results demonstrate the QGKs effective solution for 
embedding high-dimensional data into few-qubit systems, 
through expressive generator-grouped unitaries.
Rather than targeting near-term execution on very large quantum devices, QGKs are designed to maximize learning capacity under limited qubit budgets while remaining compatible with future hardware advances.
This positions the QGK as a flexible kernel framework bridging present-day constraints and longer-term fault-tolerant quantum execution.

\paragraph{Limitations}
Allthough generator-based quantum kernels are theoretically well-defined for arbitrary system sizes, heir native execution on current hardware remains limited.
Combined with today’s limited quantum device capabilities, this constrains their immediate large-scale deployment. 
To assess near-term viability, we compiled QGK circuits to current IBM hardware and simulated them under realistic noise. 
For small-scale tasks, compiled QGK circuits remain well below 100 gates and demonstrate superior noise  robustness relative to gate-based approaches. For larger-scale datasets, however, compiled depths exceed the capabilities of present-day noisy devices. In this regime, efficient tensor-based classical implementations combined with input compression provide a practical execution pathway until fault-tolerant quantum hardware capable of supporting deep generator-based embeddings becomes available. 
At the same time, these limitations highlight opportunities for future hardware–algorithm co-design aimed at native support for generator-based quantum models.

\paragraph{Outlook}
For small- and medium-scale tasks, hybrid execution offers a practical path: classical preprocessing can reduce dimensionality before quantum embedding, enabling robust performance on today’s noisy devices. 
For large-scale datasets, efficient tensor-based implementations provide a tractable classical alternative, keeping generator-based kernels competitive, even outperforming classical baselines, until quantum hardware matures. 
In the long term, with the advent of fault-tolerant systems, QGKs could be executed fully quantum, including variational training of projections and native generator-based operations.
Overall, QGK provides a scalable, expressive, and classically efficient kernel method for near-term hybrid deployment, while paving the way toward a generator-based paradigm of quantum-native learning in future hardware generations.

\subsection*{Reproducibility Statement}
We have taken several measures to ensure reproducibility of our results. A detailed description of the proposed Quantum Generator Kernel (QGK) method, including generator construction, grouping procedure, and training pipeline, is provided in the main text and Appendix \ref{app:generator}, \ref{app:grouping}, \ref{app:grouping_analysis}, and \ref{app:linear_projection}.
Theorems and proofs of the algebraic properties are included in \autoref{app:complexity}. Hyperparameters, datasets, and experimental setups are reported in \autoref{sec:eval} and \autoref{app:grouping_hp} and summarized in Tables~\ref{tab:hyperparameters} and \ref{tab:result_summary}. Noise simulations and hardware compilation details are given in \autoref{app:analysis}, with compiled depths explicitly reported. All datasets used (\texttt{moons}, \texttt{circles}, \texttt{bank}, \texttt{MNIST}, and \texttt{CIFAR10}) are publicly available, with preprocessing steps documented in the supplementary materials. For robustness, we report averages over eight random seeds. The full implementation is available upon request.

\subsection*{Impact Statement}
This paper considers the construction and applicability of generator-based quantum kernels. 
We anticipate that advances in quantum machine learning, and in particular the new paradigm of generator-based quantum kernel methods introduced here, may broaden the applicability of kernel learning and enable scalable use of quantum resources in future AI systems.
We do not expect any societal consequences from our work beyond those posed by quantum machine learning and quantum computing in general.

\bibliography{QGK}
\bibliographystyle{iclr2026}

\appendix
\section{Deriving a Universal Set of Generators}\label{app:generator}
This appendix details the construction of the generator set $\mathfrak{H}$ used throughout this work and clarifies its relation to the generalized Gell--Mann matrices, which form a canonical basis of the Lie algebra $\mathfrak{su}(N)$.
The SU(N) group includes all unitary $N \times N$ matrices under multiplication. As a Lie group, the group elements form a smooth manifold and tangential space, the Lie algebra $\mathfrak{su}(N)$. 
The Lie algebra itself is defined as an additive real vector space characterized together with a commutator relation, satisfying the Jacobi identity. 
Further, it exhibits the same dimension as the respective Lie group. The direct connection between elements of a Lie group G and elements of the corresponding Lie algebra $\mathfrak{g}$ can be defined by:
\begin{align}\forall_{X \in \mathfrak{g} } \quad e^{-t \cdot X} \in G \quad \text{with } t \in \mathbb{R}  \label{eq: group-algebra-dependency}\end{align}
However, $X \in \mathfrak{g}$ just holds as long as $ \forall_{t \in \mathbb{R}} e^{-t \cdot X} \in G$.
The Lie algebra itself constitutes a vector space which is spanned by a number of base elements $e_{i}$, while $i$ denotes the respective dimension. The number of base elements is equivalent to the dimension of the Lie algebra.
Using the fact that the Lie group G represents a differentiable manifold, the associated Lie algebra $\mathfrak{g}$ represents the tangent space to G at its identity element. This picture leads to the relation given in Eq.~\eqref{eq: tangentspacederivative}, while the indices $l$ and $k$ denote the application of the equation on a matrix element at the respective position.
\begin{align}
  \frac{\partial}{\partial{x_{i}}}A_{lk}(x_{1},x_{2}...x_{n})\bigg|_{x_{1} = 0,x_{2} = 0...x_{n} = 0} = (e_{i})_{lk}\quad A \in G \label{eq: tangentspacederivative}
\end{align}

Since it can be shown that the set of traceless square anti-Hermitian $N \times N$ matrices constitutes the $\mathfrak{su}(N)$ algebra, a complete base set $\{e_{1}, e_{2} ... \}$ of linear independent representatives of $\mathfrak{su}(N)$ is sufficient to create any element within this algebra. By making use of the relation given in Eq.~\eqref{eq: group-algebra-dependency}, we can map any element of $\mathfrak{su}(N)$ to an element of SU(N). 
Additionally, these anti-Hermitian base elements $e_{i}$ can be converted into so called Hermitian generator elements $h_{i}$ using: 
\begin{align}e_{i}= - \frac{i}{2} h_{i}\label{eq: antitohermitian}\end{align}

The connection between algebra and group via the algebra is exact \citep{mansky_scaling_2025}. 
It can also be used to build quantum circuits from Lie algebra elements, as every element of the Lie group has a corresponding quantum circuit element \citep{mansky_decomposition_2023}. 
In general, this representation of the group element as a quantum circuit is difficult to find and generally requires $4^\eta$ operations to represent \citep{bergholm_quantum_2005, shende_synthesis_2005}. 
With the choice of a particular basis, this approach can be simplified. The Pauli basis $\Pi = \bigotimes_i^\eta\{\sigma_x, \sigma_y, \sigma_z, \mathbb{I}\}\backslash \mathbb{I}^\eta$ is the discrete group of Pauli strings, the tensor product of Pauli matrices. 
For a system of $\eta$ qubits, the associated Hilbert space has dimension $H = 2^\eta$.
Any data-dependent unitary feature map acting on this space is generated by Hermitian operators in $\mathfrak{su}(N)$, the Lie algebra of traceless Hermitian $H\times H$ matrices.
In the qubit setting ($H=2$), the Pauli matrices provide a basis of $\mathfrak{su}(2)$. For higher-dimensional systems, this role is played by the generalized Gell--Mann matrices, which extend the Pauli and (three-dimensional) Gell--Mann matrices to arbitrary dimension $N$ \cite{gellmann1962symmetries}.

Due to the fact that the group of traceless square anti-Hermitian $N \times N$ matrices constitutes the $\mathfrak{su}(N)$ algebra, a complete set of base elements can be easily formulated. Through the multiplication with a complex factor, Hermitian generators can be constructed following Eq.~\eqref{eq: antitohermitian}. This also makes it possible to start directly by creating a set of base elements for Hermitian matrices (generators) and convert them back to anti-Hermitian matrices by applying the factor. This can be done during the transformation into a unitary displayed in Eq.~\eqref{eq:unitary}.

Firstly, the set of generators $\mathfrak{H}$ needs to show linear independence among all its elements. This can be satisfied if every generator contains at least one element $a_{lk} \neq 0$ for which all other generators show $a_{lk} = 0$. To further ensure that the found generators actually form a basis for Hermitian matrices, any pair of generators $h_{i}$ and $h_{j}$ needs to satisfy the following condition:
\begin{align}\label{eq:orthogonality}\Tr(h_{i}, h_{j}) = 2\cdot\delta_{ij}\end{align}
Finally, it needs to be ensured that all generators $h_{i}$ and $h_{j}$ fulfill the commutator condition characterizing all basis elements of a Lie algebra: 
\begin{align}
[h_{i}, h_{j}] = \sum_{k = 1}^{n} C_{ikl} \cdot 2i\cdot\ h_{k} \;\;\;\; with \;\; \forall_{k \in \{1,2..n\}}h_{k} \in \mathfrak{H} \;\;\wedge \;\; \forall_{i,k,l}  C_{ikl} \in \mathbb{R}
\end{align}

Based on these three conditions, 
the full generator set $\mathfrak{H} = \left\{ \mathfrak{H}^{(S)}, \mathfrak{H}^{(A)}, \mathfrak{H}^{(D)}, \right\}$ can be defined from 
the following generator subsets given in Eqs.~\eqref{eq: generators1}, \eqref{eq: generators2}, and \eqref{eq: generators3}. 
Let $E_{rc}\in\mathbb{C}^{N\times N}$ denote the matrix unit with a single nonzero entry: $(E_{rc})_{ij} = \delta_{ir}\delta_{jc}$, with $r,c\in\{0,\dots,N-1\}$.
These matrices form a basis of $\mathbb{C}^{d\times d}$ and allow a compact definition of all generalized Gell--Mann generators.
For each pair of distinct indices $0\le r<c\le N-1$, we define two Hermitian, traceless matrices $h^{(S)}_{rc} = E_{rc} + E_{cr}$ (\autoref{eq: generators1}) and $h^{(A)}_{rc} = -i\,(E_{rc} - E_{cr})$ (\autoref{eq: generators2}) that correspond to the symmetric and antisymmetric off-diagonal generators of the generalized Gell--Mann basis:

\begin{align}
\mathfrak{H}^{(S)} &=
\begin{pmatrix} 0 & 1 & & 0 \\ 1 & 0 & & \\ & & \ddots & \\ 0 &  &  & 0 \\ \end{pmatrix},\dots,
\begin{pmatrix} 0 & 0 & & 1 \\ 0 & 0 & & \\ & & \ddots &\\ 1 & &  & 0 \end{pmatrix},\dots,
\begin{pmatrix} 0 & & & 1 \\ & \ddots & &\\ & & 0 & 1 \\ 0 &  & 1 & 0 \end{pmatrix}
\label{eq: generators1}
\\
\mathfrak{H}^{(A)} &=
\begin{pmatrix}
        0 & -i & & 0 \\ i & 0 & & \\ & & \ddots & \\ 0 &  &  & 0 \\
    \end{pmatrix},\dots,
    \begin{pmatrix}
        0 & 0 & & -i \\ 0 & 0 & & \\ & & \ddots &\\ i & &  & 0
    \end{pmatrix},\dots,
    \begin{pmatrix}
        0 & & & 0 \\ & \ddots & &\\ & & 0 & -i \\ 0 &  & i & 0
    \end{pmatrix}
    \label{eq: generators2}
\end{align}

These matrices generalize the $X$- and $Y$-Pauli operators. They generate coherent transitions between computational basis states $\ket{r}$ and $\ket{c}$ and account for all off-diagonal degrees of freedom in $\mathfrak{su}(N)$.
For the construction of the generators containing diagonal elements, the traceless condition must be regarded.
Following the canonical construction of the diagonal (Cartan) generators of the generalized Gell--Mann basis of $\mathfrak{su}(N)$, this yields $N-1$ linearly independent diagonal generators 
with uniform Hilbert--Schmidt norm (\autoref{eq: generators3}).

\begin{align}
h^{(D)}_{k}&=\sqrt{\frac{2}{k(k+1)}}\left(\sum_{m=0}^{k-1} E_{mm}-k\,E_{kk}\right)
, \quad k \in \{1,\dots,N-1\}
\nonumber\\
\mathfrak{H}^{(D)} &= 
\begin{pmatrix}
1 &        &        &        \\
  & -1     &        &        \\
  &        & \ddots &        \\
  &         &        & 0
\end{pmatrix}, \dots, \sqrt{\frac{2}{(N-1)N}}
\begin{pmatrix}
1 &        &        &        \\
  & \ddots &        &        \\
  &       & 1     &        \\
  &         &        & N-1
\end{pmatrix}
\label{eq: generators3}
\end{align}

The normalization factor $\sqrt{2/(k(k+1))}$ ensures that all generators satisfy the orthogonality condition (\autoref{eq:orthogonality}), consistent with the Lie-algebra relations used throughout this work.

Algorithmically, the set of generators can be constructed with computational complexity $\mathcal{O}(4^\eta)$ via:

\begin{algorithm}
\caption{Construction of $\mathfrak{su}(2^\eta)$ Generators}\label{alg:generator_construction}
\begin{algorithmic}[1]
\REQUIRE Number of qubits $\eta$
\ENSURE Set of generators $\mathfrak{H} \gets \emptyset$ \hfill{\color{gray}$\vartriangleright$ Initialize empty generator set.}

\STATE $H \gets 2^\eta$ \hfill{\color{gray}$\vartriangleright$ Compute Hilbert space dimension.}

\STATE{\textbf{for} $(r,c) \in \{(r,c)\mid 0 \le r < c \le H-1\}$ \textbf{do} \hfill{\color{gray}$\vartriangleright$ Off-diagonal generators.}}
    \begin{ALC@g}\STATE $\mathfrak{H}\gets\mathfrak{H}\cup\{E_{rc}+E_{cr}\}$ \hfill{\color{gray}$\vartriangleright$ Symmetric (\autoref{eq: generators1})}

    \STATE $\mathfrak{H}\gets\mathfrak{H}\cup\{-i(E_{rc}-E_{cr})\}$ \hfill{\color{gray}$\vartriangleright$ Antisymmetric (\autoref{eq: generators2})}\end{ALC@g}
\STATE{\textbf{end for}}

\STATE{\textbf{for} $k \in \{1,2,\dots,H-1\}$ \textbf{do} \hfill{\color{gray}$\vartriangleright$ Diagonal (Cartan) generators.}}
\begin{ALC@g}
  \STATE $\mathfrak{H}\gets\mathfrak{H}\cup\left\{\sqrt{2/(k(k+1))}\left(\sum_{m=0}^{k-1}E_{mm} - kE_{kk}\right)\right\}$\hfill{\color{gray}$\vartriangleright$ Traceless diagonal (\autoref{eq: generators3})}
\end{ALC@g}
\STATE{\textbf{end for}}
\end{algorithmic}
\end{algorithm}

\clearpage

\section{Generator Grouping Approaches} \label{app:grouping}

\autoref{alg:generator_construction} yields a complete set of $|\mathfrak{H}| = 4^\eta - 1$ Hermitian generators spanning $\mathfrak{su}(2^\eta)$. 
Since directly assigning one parameter per generator is computationally prohibitive for growing $\eta$, we aggregate generators into \emph{Variational Generator Groups} (VGGs), each combining multiple generators into a single effective operator.
The number of groups $g$ is determined via a grouping factor $\Gamma_\eta$, which controls the number of generators per group and scales with the system size:
\begin{align}\label{eq:generators_per_group}
g = \frac{|\mathfrak{H}|}{\Gamma_\eta}, \qquad
\Gamma_\eta =
\begin{cases}
2\,\Gamma_{\eta-1} + 1, & \text{if $\eta>2$ and $\eta$ is odd},\\
2\,\Gamma_{\eta-1} - 1, & \text{if $\eta>2$ and $\eta$ is even},\\
1, & \text{if $\eta\leq 2$}.
\end{cases}
\end{align}
This construction yields an approximately exponential increase in group size with $\eta$, balancing expressivity and computational tractability, and forms the basis for the VGGs used.

One alternate approach to form a single unitary matrix incorporating all $4^\eta-1$ parameters assigns every free parameter to a specific generator of the underlying $\mathfrak{su}(N)$ algebra. By forming one large linear combination of generators $h_{i}$ multiplied by the respective parameters $\vphi_{i}$, one Hermitian operator $\hat{\mH}$ is produced:
\begin{align}\hat{\mH} = \sum_{i = 1}^{4^\eta-1} \vphi_{i} \cdot h_{i}\quad,\end{align}
while $\eta$ again corresponds to the number of qubits used.
This Hermitian operator $\hat{\mH}$ can then be converted to a single unitary matrix $\hat{\mU}$: 
\begin{align}\hat{\mU} = e^{-i\hat{\mH}} = e^{-i  \sum_{i = 1}^{4^\eta-1} \vphi{i} h_{i}}\end{align}

An important disadvantage of this approach, however, is the inability to form a finite expression of the derivative of $\hat{\mU}$ with regard to one of the involved free parameters $\vphi{i}$ in the general case. This issue becomes apparent when differentiating $\hat{\mU}$ while displaying the involved exponential function in its series representation. 
The difficulties in forming a finite expression for the derivative originate from the non-commuting characteristic present between specific generators $h_{i}$ and $h_{j}$. An example of this is given in Eq.~\eqref{eq: non_commuting_generators}:
\begin{align}\hat{\mH} = \sum_{i = 1}^{4^\eta-1} \vphi_i\cdot h_k \quad \forall\quad h_k \in \mathcal{G}_i
\label{eq: non_commuting_generators}\end{align}
To circumvent the problem, we chose another approach, which enables the formulation of a finite expression of the derivative and which is described in \autoref{sec:vgg} in more detail. 
If we separate the set of generators into clusters, assigning one free parameter to every cluster respectively, we generate several unitary sub-operators $\hat{\mU}_i$, which we contract to a single one according to Eq.~\eqref{eq:unitary}. Since every sub-operator contains only one free parameter, no commutation is necessary in order to form a finite derivative; hence, the non-commuting characteristic does not cause any problems.
Also, as given by the following proof showing that the unitary matrices $\mU$ representing the transformation of a quantum state form a closed group regarding matrix multiplication, the resulting operator is also unitary:
\begin{align}
    \hat{\mU} &= \hat{\mU}_{\vphi_{1}}\hat{\mU}_{\vphi_{2}}\dots\hat{\mU}_{\vphi_{p}}\quad \cap\quad\forall_{i \epsilon\{1 \dots p\}} \;\; \hat{\mU}_{\vphi_{i}}^{+}\hat{\mU}_{\vphi_{i}} = \mathbb{I} \nonumber\\
    \Longrightarrow \hat{\mU}^{+}\hat{\mU} &= (\hat{\mU}_{\vphi_{1}}\hat{\mU}_{\vphi_{2}}\dots\hat{\mU}_{\vphi_{p}})^{+} \;(\hat{\mU}_{\vphi_1}\hat{\mU}_{\vphi_2}\dots\hat{\mU}_{\vphi_g}) \nonumber\\
    &= \;\hat{\mU}_{\vphi_g}^{+}\dots\hat{\mU}_{\vphi_2}^{+}\hat{\mU}_{\vphi_1}^{+} \cdot \hat{\mU}_{\vphi_1}\hat{\mU}_{\vphi_2}\dots\hat{\mU}_{\vphi_g}\nonumber\\
    &= \;\hat{\mU}_{\vphi_g}^{+}\dots\hat{\mU}_{\vphi_2}^{+}\cdot \mathbb{I} \cdot \hat{\mU}_{\vphi_2}\dots\hat{\mU}_{\vphi_g}
    = \;\mathbb{I} \label{eq: closed_matmul_group}
\end{align}
However, maximizing the number of parameters incorporated in $\hat{\mU}$, i.e., introducing one parameter per generator, requires clusters containing only a single generator, respectively. As a result, we would create one sub-operator per generator. This means we have to do a matrix multiplication for every additional generator in order to get the contracted unitary $\hat{\mU}$. This again sets a challenge to the approach since the generator number increases exponentially ($4^\eta-1$) with $\eta$ qubits, making the calculation of $\hat{\mU}$ computationally expensive. This justifies the implemented flexibility of choosing the size of the used clusters, enabling the adjustment of the tradeoff between computational expense and size of the introduced context.

\clearpage

\section{Grouping Hyperparameter Analysis}\label{app:grouping_hp}

The \textit{Variational Generator Group} (VGG) construction (Alg.~\ref{alg:vgg_construction}) introduces two key hyperparameters:
(1) the number of generators per group, controlled by the number of groups $g$ (cf., Eqs.~\ref{eq:groups} and \ref{eq:generators_per_group}), and
(2) the projection width $w$, which determines how generators are assigned to groups (wide, medium, or narrow stride).
This section analyzes how these hyperparameters influence the theoretical properties of the QGK, including entanglement capacity and expressibility (Fig.~\ref{fig:app:grouping}, \autoref{tab:app:expressibility}), compiled depth (Tab.~\ref{app:tab:group_depths}), and their empirical impact on downstream learning performance (Tab.~\ref{app:tab:grouping_hp}).
Fig.~\ref{fig:projection comparison} shows a visual comparison between wide and narrow projections widths.

\begin{figure}[ht]
\includegraphics[width=0.775\linewidth]{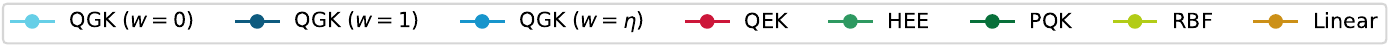}\vspace{-1em}\\
\subfloat[Entanglement]{\includegraphics[width=0.775\linewidth]{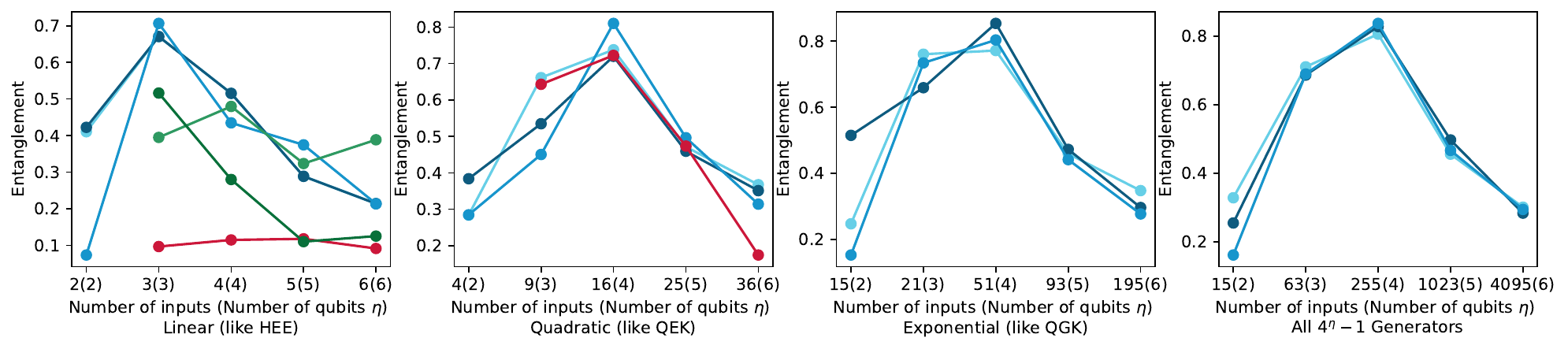}\label{fig:app:entanglement}} 
\subfloat[Projection Heatmaps]{\includegraphics[width=0.225\linewidth]{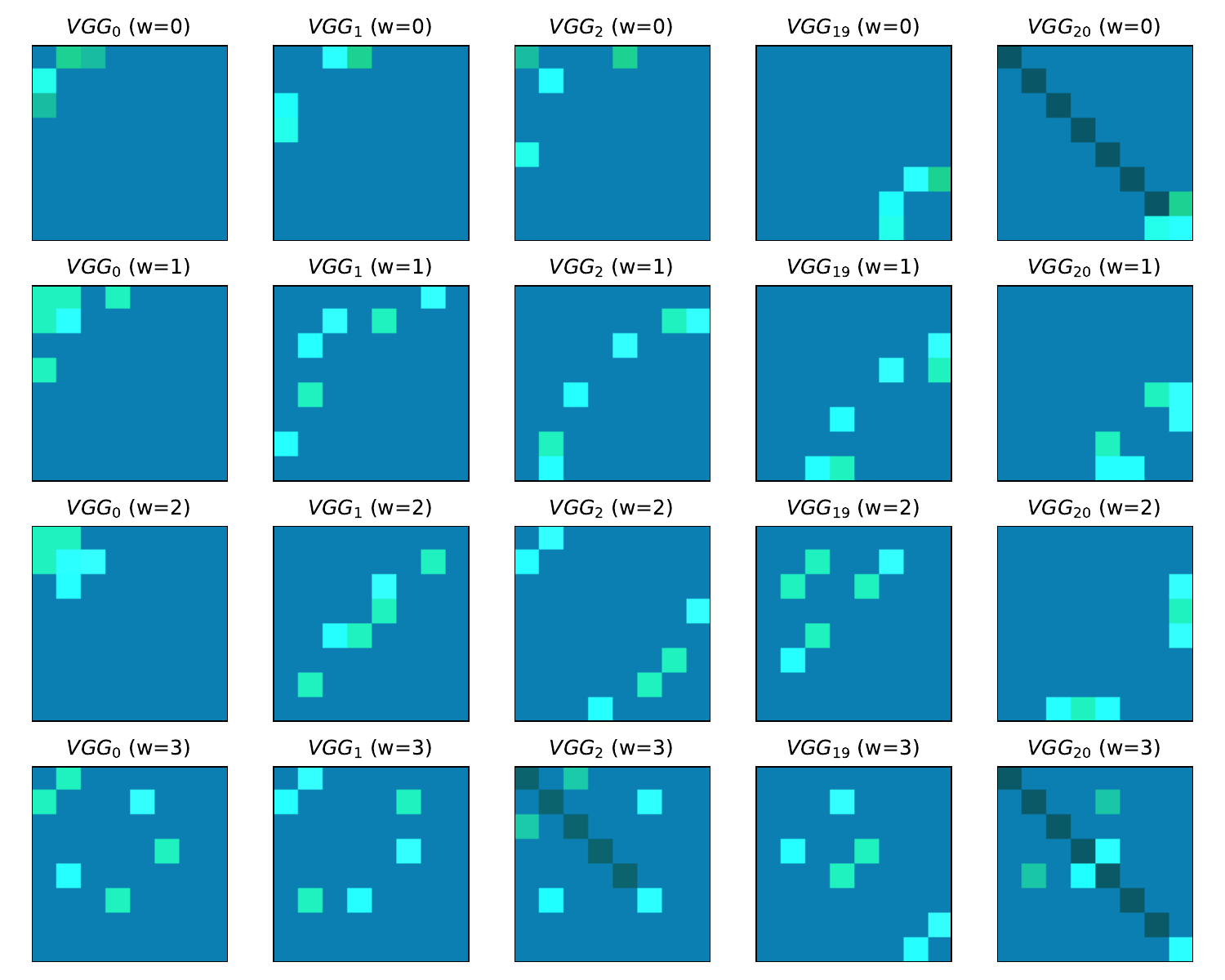}\label{fig:projection comparison}}\\
\subfloat[Expressibility]{\includegraphics[width=0.775\linewidth]{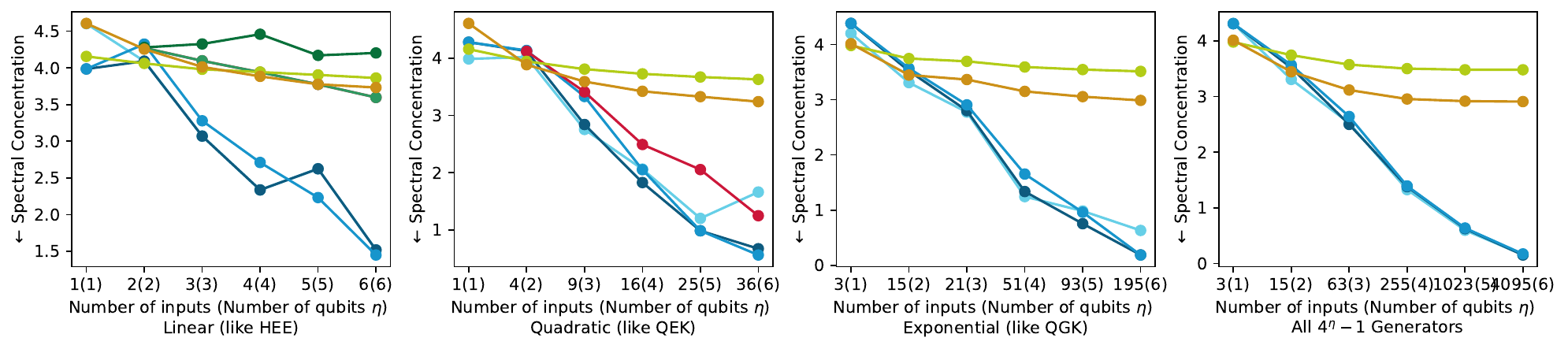}\label{fig:app:expressibility}}
\subfloat[Group Size]{\includegraphics[width=0.225\linewidth]{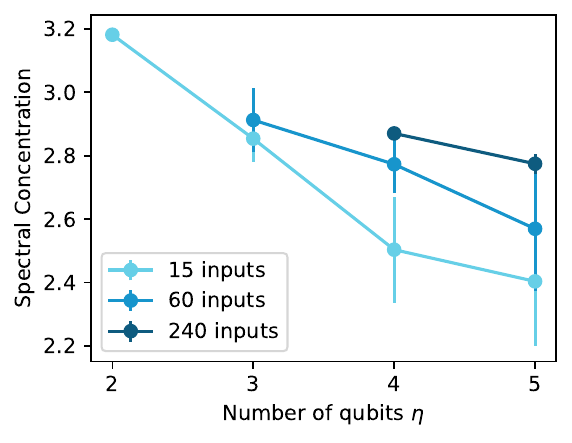}\label{fig:app:groupsize_exp}}
\caption{VGG grouping analysis and comparison regarding 
\protect\subref{fig:app:entanglement} the entanglement capability by means of the Meyer-Wallach measure and \protect\subref{fig:app:expressibility} the
expressibility by means of the spectral concentration, accross four input scalings: \textit{linear} ($g=\eta$, left) yielding the most generators per group, \textit{quadratic} ($g=\eta^2$), \textit{exponential} ($g$ according to Eqs.~\ref{eq:groups} and \ref{eq:generators_per_group}), and \textit{all} (right) using the maximum available input dimensions, i.e., the total number of generators $4^\eta-1$ and three projection widths: the default $w=\eta$ causing a wide stride, i.e., assignment of distant generators to groups, $w=1$, causing a medium stride, and $w=0$, causing a narrow stride. \protect\subref{fig:app:groupsize_exp} shows a comparison of spectral concentration across group sizes $g\in[15,60,240]$, averaged over projection widths ($w \in \{0,\,1,\,\eta\}$), with error bars indicating the standard deviation.
\protect\subref{fig:projection comparison} shows projection heatmaps for VGG$_{0,1,2,19,20}$ for $\eta=3$ qubits, with widths from $w=0$ (upper) to $w=\eta$ (lower), visualizing the magnitude (blue) and phase (green) of the resulting generator matrices $\hat{\mH}$.
}\label{fig:app:grouping}
\end{figure}

\paragraph{Entanglement Capacity}
\autoref{fig:app:entanglement} shows that QGK maintains stable entanglement capabilities across all examined projection widths and grouping strategies.
Variance decreases with increasing numbers of qubits and groups, demonstrating robustness as system size grows.

\paragraph{Expressivity} \autoref{fig:app:expressibility} shows the expressibility of the QGK behaves overall robustly across different projection widths and numbers of qubits, particularly when using the suggested exponential scaling.
Finally, \autoref{fig:app:groupsize_exp} and \autoref{tab:app:expressibility} summarize the effect of different grouping configurations on kernel expressivity, showing three key insights:
(1) QGK expressivity is highly robust across projection widths, with only minor variation between $w \in \{0, 1, \eta\}$, even as the number of qubits and group sizes change.
(2) Expressivity increases systematically with system size, reflecting the growing representational capacity of larger generator sets.
(3) Coarser groupings (i.e., fewer groups with more generators per group) consistently yield higher expressivity, indicating that merging diverse generators enhances the attainable unitary transformation.

\begin{table}[ht]\centering\scriptsize
\begin{tabular}{|c|c|c|c|c|c|c|c|c|}\hline
Scaling & Width & $\eta={1}$ & $\eta={2}$ & $\eta={3}$ & $\eta={4}$ & $\eta={5}$ & $\eta={6}$ & Total $\downarrow$\\
\hline \multirow{4}{*}{Quadratic (like QEK)} &  & 1 inputs & 4 inputs & 9 inputs & 16 inputs & 25 inputs & 36 inputs & -\\ 
 & $0$ & $0.19$ & $0.07$ & $0.22$ & $0.08$ & $0.14$ & $0.70$ & $1.41$ \\
 & $1$ & $0.10$ & $0.03$ & $0.13$ & $0.15$ & $0.07$ & $0.29$ & $0.78$ \\
 & $\eta$ & $0.10$ & $0.04$ & $0.35$ & $0.07$ & $0.07$ & $0.40$ & $1.04$ \\
\hline \multirow{4}{*}{Exponential (like QGK)} &  & 3 inputs & 15 inputs & 21 inputs & 51 inputs & 93 inputs & 195 inputs & -\\ 
 & $0$ & $0.12$ & $0.16$ & $0.05$ & $0.17$ & $0.08$ & $0.30$ & $0.88$ \\
 & $1$ & $0.06$ & $0.05$ & $0.03$ & $0.08$ & $0.15$ & $0.15$ & $0.51$ \\
 & $\eta$ & $0.06$ & $0.10$ & $0.08$ & $0.24$ & $0.06$ & $0.15$ & $0.69$ \\
\hline \multirow{4}{*}{All $4^\eta-1$ Generators} &  & 3 inputs & 15 inputs & 63 inputs & 255 inputs & 1023 inputs & 4095 inputs & -\\ 
 & $0$ & $0.00$ & $0.16$ & $0.03$ & $0.04$ & $0.02$ & $0.00$ & $0.26$ \\
 & $1$ & $0.00$ & $0.05$ & $0.06$ & $0.01$ & $0.01$ & $0.01$ & $0.14$ \\
 & $\eta$ & $0.00$ & $0.10$ & $0.09$ & $0.03$ & $0.02$ & $0.01$ & $0.25$ \\
\hline\end{tabular}
\caption{Sensitivity of kernel expressivity to projection width: Absolute deviations $|s_{w,\eta} - \bar{s}_\eta|$ of kernel expressivity scores $s$ across projection widths $w \in \{0, 1, \eta\}$, evaluated for $\eta$ qubits under different input scalings (i.e., varying the number of generators per group).}\label{tab:app:expressibility}
\end{table}

\newpage
\paragraph{Compiled Circuit Depth}
\autoref{app:tab:group_depths} reports the compiled depths on IBM Falcon (\textit{FakeToronto}, $\eta\leq 27$).
Across all projection widths and input scaling rules, depth variations are negligible.
This is expected because the underlying set of generators remains identical; only their groupings differ.
Thus, VGG rearrangements do not materially affect the gate count for fixed hardware constraints.

\begin{table}[ht]\centering\scriptsize
\begin{tabular}{|c|c|c|c|c|c|}\hline
Dataset & Groups & Generators per Group & $w=1$ & $w=\eta$ & $w=2\eta$ \\ \hline
\multirow{3}{*}{\shortstack{\texttt{moons}\\($d=2$, $\eta=2$)}} 
& Q ($4$) & $3$ & $28$ & $28$ & $28$ \\
& g ($15$) & $1$ & $28$ & $28$ & $28$ \\ 
& $|\mathfrak{H}|$ ($15$)& $1$ & see above & see above & see above \\ \hline 

\multirow{3}{*}{\shortstack{\texttt{MNIST} \\($d=784$, $\eta=5$)}} 
& Q ($25$) & $40$ & $4756$ & $4748$ & $4755$ \\
& g ($93$) & $11$ & $4752$ & $4755$ & $4759$ \\ 
& $|\mathfrak{H}|$ ($1023$) & $1$ & $4755$ & $4753$ & $4756$ \\ \hline 

\end{tabular}
\caption{Level-1-optimized circuit depths compiled to a 5-qubit IBM Falcon device for different generator grouping strategies: Quadratic ($Q = \eta^2$, matching QEK capacity), Exponential ($g = 3(2^\eta - 2(\eta \bmod 2) + 1)$, as per \autoref{eq:groups}), and All ($|\mathfrak{H}| = 4^\eta - 1$, i.e., one group per generator).}\label{app:tab:group_depths}
\end{table}

\paragraph{Performance variation}
\autoref{app:tab:grouping_hp} summarizes downstream classification accuracy under different grouping configurations for representative benchmarks (\texttt{moons}, \texttt{bank}, \texttt{MNIST}).
On small datasets (e.g., \texttt{moons}), quadratic and exponential grouping choices perform comparably, while linear grouping performs worst, when pre-training the kernel's projection to maximize KTA.
On higher-dimensional datasets, exponential scaling yields consistently strong performance, with only small deviations between projection widths.
Full-generator scaling, while theoretically appealing, offers only marginal gains over exponential scaling, at significantly higher classical simulation costs.

\begin{table}[ht]\centering\scriptsize
\begin{tabular}{|c|c|c|c|c|c|}\hline
Dataset & Groups / Inputs & $\Gamma$ & $w=0$ & $w=1$ & $w=2$\\ \hline
\multirow{3}{*}{\shortstack{\texttt{moons} \\ ($d=2$, $\eta=2$, $n=200$)}}
& $\eta=2$ & 7 & $0.90 \pm 0.04$ & $0.91 \pm 0.05$ & $0.90 \pm 0.04$ \\
& $Q=4$ & 3 & $0.95 \pm 0.03$ & $0.95 \pm 0.03$ & $0.95 \pm 0.03$ \\
& $|\mathfrak{H}|=g=15$ & 1 & $0.96 \pm 0.04$ & $0.96 \pm 0.04$ & $0.96 \pm 0.04$ \\ \hline
\multirow{3}{*}{\shortstack{\texttt{bank} \\ ($d=16$, $\eta=2$, $n=200$)}}
& $\eta=2$ & 7 & $0.77 \pm 0.05$ & $0.75 \pm 0.10$ & $0.77 \pm 0.05$ \\
& $Q=4$ & 3 & $0.79 \pm 0.07$ & $0.82 \pm 0.07$ & $0.79 \pm 0.07$ \\
& $|\mathfrak{H}|=g=15$ & 1 & $0.88 \pm 0.06$ & $0.88 \pm 0.06$ & $0.88 \pm 0.06$ \\ \hline
\multirow{3}{*}{\shortstack{\texttt{MNIST} \\ ($d=784$, $\eta=5$, $n=1000$)}}
& $Q=25$ & 40 & $0.85 \pm 0.03$ & $0.83 \pm 0.02$ & $0.85 \pm 0.03$ \\
& $g=93$ & 11 & $0.88 \pm 0.03$ & $0.88 \pm 0.03$ & $0.88 \pm 0.03$ \\
& $|\mathfrak{H}|=1023$ & 1 & $0.91 \pm 0.03$ & $0.91 \pm 0.03$ & $0.91 \pm 0.03$ \\
\hline\end{tabular}
\caption{Empirical Hyperparameter Sensitivity of QGK: Final classification accuracy (mean $\pm$ error margin) on a subset of three selected $d$-dimensional benchmarks, with $\Gamma$ generators per group, comparing different projection widths; wide ($w=1$, grouping distant generators), medium ($w=\eta$), and narrow ($w=2\eta$, grouping nearby generators), across different group scaling regimes relative to the number of qubits $\eta$: Linear ($\eta$ groups), Quadratic ($Q=\eta^2$, comparable to QEK capacity), Exponential ($g=3(2^\eta - 2(\eta \bmod 2) + 1)$, cf. \autoref{eq:groups}), and All ($|\mathfrak{H}|{=}4^\eta{-}1$, one group per generator).}\label{app:tab:grouping_hp}
\end{table}

This confirms that the default exponential grouping rule provides a robust, well-balanced choice for both theoretical coverage and empirical accuracy.

\clearpage

\section{Theoretical Characterization of Grouping Robustness}\label{app:grouping_analysis}

We begin by analyzing the grouping strategy underlying the construction of the \textit{Variational Generator Groups} (VGGs), which serve as building blocks of the \textit{Quantum Generator Kernel} (QGK). A VGG is defined as a parameterized Hermitian operator constructed from a subset of a full Hermitian generator basis $\mathfrak{H}$:
\begin{equation}
    \hat{\mH}_i = \sum_{j \in \mathcal{G}_i} h_j, \quad h_j \in \mathfrak{H}, \quad \mathcal{G}_i \subset \mathfrak{H},
\end{equation}
where $\{ \mathcal{G}_1, \dots, \mathcal{G}_g \}$ forms a collection of groups. The full generator set $\mathfrak{H}$ has size $|\mathfrak{H}| = 4^\eta - 1$ for $\eta$ qubits, and the VGGs are composed into a parameterized unitary via:
\begin{equation} 
    \hat{\mU}_{\vphi} = \exp\left(\sum_{i=1}^g -i\cdot\vphi_i\cdot\hat{\mH}_i\right).
\end{equation}
Thus, each grouping defines a new operator basis $\mathfrak{H}' = \{\hat{\mH}_1, \dots, \hat{\mH}_g\}$ where each $\hat{\mH}_i$ is a structured linear recombination of elements from $\mathfrak{H}$. Provided that the grouped operators remain linearly independent and collectively span the same subspace as $\mathfrak{H}$, the QGK remains expressively equivalent. 

\begin{proposition}[Expressive Robustness under Generator Grouping]
Let $\mathfrak{H}, \mathfrak{H}'$ be two generator sets with $\mathrm{span}(\mathfrak{H}') = \mathrm{span}(\mathfrak{H})$, and $\mathfrak{H}'$ formed via linear combinations over $\mathfrak{H}$. Then, for any QGK defined over $\mathfrak{H}'$, there exists an equivalent QGK defined over $\mathfrak{H}$ up to reparameterization. Consequently, the expressive capacity of the kernel is preserved under such groupings.
\end{proposition}

\paragraph{Robustness Conditions}
Let us now state the necessary conditions for robustness more precisely:

\begin{definition}[Robust Grouping Criteria]
A generator grouping $\{\mathcal{G}_i\}_{i=1}^{g}$ is robust if:
\begin{enumerate}[left=4pt, topsep=-2pt]
    \item Each generator $h_j \in \mathfrak{H}$ is assigned to exactly one group $\mathcal{G}_i$ (i.e., the groups form a partition).
    \item The subspace $\mathrm{span}(\{ \hat{\mH}_1, \dots, \hat{\mH}_g \})$ approximates or equals $\mathrm{span}(\mathfrak{H})$.
    \item The grouped operators $\mathfrak{H}' = \{ \hat{\mH}_1, \dots, \hat{\mH}_g \}$ are linearly independent.
\end{enumerate}
\end{definition}

\paragraph{Verifying the Proposed Grouping Scheme}
Our proposed grouping strategy (\autoref{alg:vgg_construction}) partitions the generator set $\mathfrak{H}$ into $g = |\mathfrak{H}| / \Gamma_\eta$ groups, with $\Gamma_\eta$ scaling approximately exponentially in $\eta$ (cf. \autoref{eq:groups}). The generators are assigned to groups using a cyclic stride determined by the projection width $w$, where the index mapping is given by: $\mathrm{idx} = \left\langle \left(j \cdot 2^w \bmod |\mathfrak{H}|, j \bmod g \right) \,\middle|\, j = 0, 1, \dots, |\mathfrak{H}| - 1 \right\rangle.$
We now argue that this scheme satisfies the robustness criteria above:
\begin{enumerate}[left=4pt,topsep=-2pt]
    \item \textbf{Partition Validity.}  
    By construction, every generator $h_j \in \mathfrak{H}$ is assigned to exactly one group via a deterministic, non-overlapping index mapping. Therefore, the grouping is a strict partition: $$\bigcup_i \mathcal{G}_i = \mathfrak{H}\quad\text{ and }\quad\mathcal{G}_i \cap \mathcal{G}_j = \emptyset \quad\text{ for }\quad i \ne j.$$

    \item \textbf{Preservation of Span.}  
    Each group $\mathcal{G}_i$ contains approximately $\Gamma_\eta$ generators. Since $\sum_i |\mathcal{G}_i| = |\mathfrak{H}|$, and all generators are used exactly once, the union of the groups spans the same space as the original generator basis. 

    \item \textbf{Linear Independence.}  
    We define the group assignment matrix $\mM \in \mathbb{R}^{D \times g}$ where $\mM_{j,i} = 1$ if $h_j \in \mathcal{G}_i$ and $0$ otherwise. 
    Since the groups $\{\mathcal{G}_i\}$ form a strict partition of the linearly independent basis $\{h_j\}$, each column of $\mM$ has support on a disjoint subset of rows. 
    Consequently, the columns of $\mM$ are linearly independent and $\mathrm{rank}(\mM)=g$, implying that the grouped operators $\{\hat{\mH}_i\}$ are linearly independent.
    The specific grouping strategy (e.g., stride or ordering) does not affect linear independence, provided each group is non-empty and the groups form a strict partition of the basis.
   
\end{enumerate}

\begin{proposition}[Sufficient Condition for Linear Independence]
Let $\mathfrak{H}$ be a linearly independent generator basis of dimension $|\mathfrak{H}|$, and let the grouping matrix $\mM \in \{0,1\}^{|\mathfrak{H}| \times g}$ have full column rank. If each group $\mathcal{G}_i$ is non-empty and the grouping forms a partition, then the grouped operators $\{ \hat{\mH}_i \}$ are linearly independent.
\end{proposition}

In our default configuration with $g = |\mathfrak{H}| / \Gamma_\eta$ (cf. \autoref{eq:groups}) and $w = 1$, the cyclic offset ensures that all generators are used once, and that groups combine uncorrelated directions. Thus, the grouping matrix $\mM$ achieves full rank in practice, ensuring linear independence of the VGGs and preservation of expressive capacity.

\paragraph{Conclusion.}
The above analysis confirms that the default grouping scheme --- using exponentially many groups ($g \sim 3 \cdot 2^\eta$ as defined in \autoref{eq:groups}) and a wide projection width ($w = 1$) --- satisfies the robustness criteria outlined above. It yields a strict partition of the full Hermitian generator basis $\mathfrak{H}$, ensuring full coverage of the operator subspace without redundancy or degeneracy. Moreover, the wide stride distributes algebraically diverse generators across groups, promoting linear independence among the grouped operators $\hat{\mH}_i$.
This explains the empirical robustness observed across different grouping strategies, which yield near-identical entanglement and expressivity characteristics as long as the induced span of $\mathfrak{H}'$ remains intact. In contrast, alternative scaling rules, such as quadratic grouping with $g = \eta^2$, may lead to underutilization of generators when $g < |\mathfrak{H}|$, breaking full coverage of $\mathrm{span}(\mathfrak{H})$ and thereby reducing the expressivity of the resulting feature map. 
Our theoretical findings thus not only justify the empirical stability of the default scheme but also explain the degradation observed in configurations that fail to meet the robustness conditions, such as narrow groupings or insufficient group counts. 
Finally, we note that even when the group count equals the generator count ($g = |\mathfrak{H}|$), the resulting set of grouped operators $\mathfrak{H}'$ is not invariant under different stride values: the ordering of generators within the VGG sequence is affected by the projection width $w$, leading to parameter-wise differences in unitary composition. While this does not impact expressivity at the algebraic level, it does influence the specific parameterization of the resulting kernel.

\subsection{Kernel Expressibility Bounds under Grouped Generators}
\label{app:expressivity_bounds}

In our empirical analysis (Appendix~\ref{app:grouping_hp}), we evaluate kernel expressivity using the KL divergence between the normalized eigenvalue spectrum of the kernel matrix and a uniform reference distribution:
\begin{equation}\mathcal{E}(K) := D_{\mathrm{KL}}\!\left(\lambda(K)\,\middle\|\,\tfrac{1}{n}\mathbf{1}\right) = \sum_{i=1}^n \lambda_i(K)\,\log\!\left(n\,\lambda_i(K)\right),\end{equation}
where $\lambda(K)$ denotes the spectrum of $K$ scaled to sum to one and $K \in \mathbb{R}^{n\times n}$ is the kernel matrix $K_{ij} = \mathcal{K}_\phi(x_i,x_j) = |\langle \psi(x_i),\psi(x_j)\rangle|^2$ evaluated on $n$ randomly sampled inputs.
Here we provide bounds on $\mathcal{E}(K)$ for the QGK under grouped generators and discuss their dependence on the grouping structure and the number of qubits~$\eta$.

\paragraph{Preliminaries}
Recall the QGK feature map
\begin{equation}
    \ket{\varphi(x)}=\exp\!\Big(-i \sum_{i=1}^{g} \phi_i(x)\, \hat{\mH}_i\Big)\ket{0},
    \qquad
    \hat{\mH}_i = \sum_{j\in \mathcal{G}_i} h_j,
\end{equation}

where $\{\mathcal{G}_i\}_{i=1}^g$ is a strict partition of the full Pauli operator basis $\mathfrak{H}$ on $\eta$ qubits into $g$ disjoint groups, each inducing a Hermitian operator $\hat{\mH}_i$, with the total number of elementary generators $|\mathfrak{H}| = 4^\eta - 1$.
Because Pauli operators form an orthonormal basis with respect to the Hilbert--Schmidt inner product, their groupings satisfy:
\begin{equation}
\|\hat{\mH}_i\|_F^2 = \Big\lVert\sum_{j\in\mathcal{G}_i} h_j\Big\rVert_F^2 = \sum_{j\in\mathcal{G}_i}\|h_j\|_F^2 = |\mathcal{G}_i|.
\end{equation}
Because the elementary generators are orthogonal under the Hilbert--Schmidt inner product and are normalized as in Eq.~(11),
\(\mathrm{Tr}(h_j h_k)=2\delta_{jk}\), we obtain
\begin{equation}
\|\hat{\mH}_i\|_F^2 = \Big\lVert\sum_{j\in\mathcal{G}_i} h_j\Big\rVert_F^2 = \sum_{j,k\in\mathcal{G}_i}\mathrm{Tr}(h_j h_k) = \sum_{j\in\mathcal{G}_i}\mathrm{Tr}(h_j^2) = 2\,|\mathcal{G}_i|.
\end{equation}
Thus, the group sizes $|\mathcal{G}_i|$ quantify the Frobenius weight of each grouped generator up to the constant factor set by the basis normalization.
Intuitively, larger groups $|\mathcal{G}_i|$ excite more Pauli directions at once, increasing the variability of the effective Hamiltonian and, in turn, influencing the dispersion of the kernel spectrum.
Balanced groupings (similar $|\mathcal{G}_i|$ across $i$) produce unitaries whose action is distributed over many Pauli directions, while unbalanced groupings (some groups very large or very small) lead to anisotropic generator structure and potentially concentrated kernels.

\paragraph{KL-Expressibility Bounds}

We now relate the expressivity metric $\mathcal{E}(K)$ to the group sizes $|\mathcal{G}_i|$.

\begin{theorem}[Expressibility Bounds for Grouped Generators]
\label{thm:bounds}
Assume (i) inputs are sampled i.i.d.\ from a bounded distribution and
(ii) the groups $\{\mathcal{G}_i\}_{i=1}^g$ form a strict partition of
$\mathfrak{H}$. 
Combining the second-moment lower bound (driven by group \emph{balance}) and the fourth-moment upper bound (driven by group \emph{anisotropy}), then there exist constants $c_1,c_2>0$, s.t.
\begin{equation}\boxed{ \frac{c_1}{n}\, \frac{\sum_i |\mathcal{G}_i|} {\bigl(\sum_i \sqrt{|\mathcal{G}_i|}\bigr)^2} \;\le\; \mathcal{E}(K) \;\le\; c_2\, \frac{\sum_i |\mathcal{G}_i|^{2}} {\sum_i |\mathcal{G}_i|} }\label{eq:bounds}\end{equation}
\label{thm:kl_bounds}
\end{theorem}

\begin{proof}[Proof sketch.]
The KL divergence $\mathcal{E}(K)=D_{\mathrm{KL}}(\lambda(K)\,\|\,\tfrac1n\mathbf{1})$ can be bounded above and below by the squared $\ell_2$ distance between the kernel spectrum $\lambda(K)$ and the uniform distribution via standard Pinsker-type inequalities.
This $\ell_2$ distance is controlled by the variance and higher-order spectral moments of $K$, which depend on quantities of the form $\mathrm{Tr}\bigl(U(x)^\dagger U(x')\bigr)^m,$ and therefore on how the Hamiltonian spreads its energy across Pauli directions.
The three structural quantities in Eq.~\eqref{eq:bounds} arise from this moment analysis:
\begin{enumerate}[label={\textbf{(\arabic*)}}, noitemsep, left=4pt,topsep=-2pt]

\item \textbf{Total generator mass: }$\sum_i |\mathcal{G}_i|$.
Since the grouped operators satisfy $\|\hat{\mH}_i\|_F^2=2|\mathcal{G}_i|$, the total Frobenius norm of the generator basis is 
$\sum_i\|\hat{\mH}_i\|_F^2 = 2|\mathfrak{H}| = 2(4^\eta-1)$.
This determines the \emph{overall amount of variance} that the Hamiltonian can inject into the kernel.  
It is therefore the natural normalisation factor in both upper and lower bounds.

\item \textbf{Group balance:} $\bigl(\sum_i \sqrt{|\mathcal{G}_i|}\bigr)^2$.
This term arises from bounding the second spectral moment of the kernel, which depends on how the total generator mass $\sum_i \|\hat{\mH}_i\|_F^2$ is distributed across groups.
By Cauchy–Schwarz, $(\sum_i \|\hat{\mH}_i\|_F)^2 \le g \sum_i \|\hat{\mH}_i\|_F^2,$ with equality if and only if all $\|\hat{\mH}_i\|_F$ are equal.
Since $\|\hat{\mH}_i\|_F^2=2|\mathcal{G}_i|$, this yields the balance factor $(\sum_i \sqrt{|\mathcal{G}_i|})^2$.
If the groups are perfectly balanced ($|\mathcal{G}_i|\equiv \Gamma$), $\sum_i \sqrt{|\mathcal{G}_i|}= g\sqrt{\Gamma}$, which maximizes this quantity and hence \emph{minimizes the lower bound}.
A small lower bound corresponds to a kernel spectrum that is close to uniform.
Thus the denominator $(\sum_i\sqrt{|\mathcal{G}_i|})^2$ encodes how evenly the Hamiltonian’s energy is spread over groups.

\item \textbf{Group anisotropy:} $\sum_i |\mathcal{G}_i|^2$.
The upper bound depends on the fourth moment of $K$, which is dominated by the largest generator norms.
Since $\|\hat{\mH}_i\|^4 = 4|\mathcal{G}_i|^2$, we obtain the anisotropy term $\sum_i |\mathcal{G}_i|^2$.
If one group is much larger than the rest, this term becomes large, and the kernel spectrum admits large spikes, increasing KL divergence.
Conversely, if all groups are equal ($|\mathcal{G}_i|\equiv \Gamma$), $\sum_i |\mathcal{G}_i|^2 = g\Gamma^2$ and $\sum_i |\mathcal{G}_i| = g\Gamma$, so their ratio is simply $\Gamma$, and the upper bound remains tightly controlled.
\end{enumerate}
Thus, the \textit{lower bound} becomes small when the group sizes are well balanced and the \textit{upper bound} becomes small only when the groups are nearly uniform, and grows when the generator structure is highly uneven.
All constant factors from higher-order expansions are absorbed into
$c_1,c_2$.
\end{proof}

\paragraph{Scaling with Qubit Number}

The bounds in Eq.~\eqref{eq:bounds} scale with the number of qubits
$\eta$ through the distribution of the group sizes $\{|\mathcal{G}_i|\}$.
In many configurations considered in this work, including
(i) the default exponential grouping (cf. Eq.~\ref{eq:groups}) and
(ii) using all generators separately,
the groups $\{\mathcal{G}_i\}_{i=1}^g$ form a strict partition of $\mathfrak{H}$ with \emph{equal} size $|\mathcal{G}_i|=\Gamma_\eta$, where $\Gamma_\eta = \lfloor (2^\eta+1)/3 \rfloor$ for exponential grouping
and $\Gamma_\eta = 1$ for the all-generators case.
In both cases we have
$$
\sum_i \sqrt{|\mathcal{G}_i|} = g\sqrt{\Gamma_\eta}, \qquad
\sum_i |\mathcal{G}_i|^2 = g\,\Gamma_\eta^2, \qquad
\sum_i |\mathcal{G}_i| = g\,\Gamma_\eta.
$$
Plugging this into Eq.~\eqref{eq:bounds} yields
\begin{equation}
\frac{c_1}{ng} \;=\; \frac{c_1}{n}\, \frac{g\,\Gamma_\eta}{g^2\Gamma_\eta} \;\le\; \mathcal{E}(K) \;\le\; c_2\, \frac{g\,\Gamma_\eta^2}{g\,\Gamma_\eta} \;=\; c_2 \Gamma_\eta
\end{equation}
showing that for balanced groupings the KL-based expressibility metric cannot vanish faster than $\mathcal{O}(1/(n g))$.
In particular, as the number of groups $g$ grows with $\eta$ (e.g.\ $g \approx 3\cdot 2^\eta$ under exponential grouping or $g=4^\eta-1$ for all generators), the lower bound shrinks but remains controlled and importantly does not grow with $\eta$.
Using all generators as independent groups asymptotically gives 
$\mathcal{E}(K)\in\Omega(4^{-\eta}/n)$ and $\mathcal{E}(K)\in\mathcal{O}(1)$,
while the default exponential grouping scheme gives $\mathcal{E}(K) \in \Omega(2^{-\eta}/n)$ and $\mathcal{E}(K)\in\mathcal{O}(2^\eta)$, showing that for balanced groupings the lower expressibility bound decays exponentially in~$\eta$.
Because KL divergence scales with the \emph{normalized} spectrum, the constants $c_1,c_2$ absorb the dimensional factors of $K$.
Note that although the upper bound scales with $\eta$ through the group size $\Gamma_\eta$, it reflects only a worst-case concentration scenario derived from moment inequalities, not the typical behaviour of the kernel. 
In practice (cf. Fig.~\ref{fig:app:grouping}), the QGK exhibits consistently more favorable scaling of $\mathcal{E}(K)$ with $\eta$, as the mixing of many non-commuting Pauli directions within each grouped generator enhances phase dispersion and suppresses spectral concentration.
Overall, exponential grouping yields a well-controlled upper bound and a rapidly decaying lower bound, providing the best balance between expressivity, stability, and parameter efficiency as $\eta$ increases.

\section{Theoretical Role of the Linear Projection in Quantum Kernel Learning}\label{app:linear_projection}

The \textit{Quantum Generator Kernel} (QGK) embeds data into quantum states via a parameterized unitary whose coefficients are produced by a classical linear affine transformation. 
This section provides a theoretical characterization of this mechanism, explains the resulting \textit{Reproducing Kernel Hilbert Space} (RKHS) structure, and contrasts QGK with hybrid quantum-classical models employing nonlinear pre- and post-processing.

\paragraph{Setup}
Given input data $x \in \mathbb{R}^d$ and $g$ \textit{Variational Generator Groups} (VGGs), the QGK constructs a data-dependent generator weighting via an affine projection:
\begin{equation}
    \vphi = \mathcal{F}_\vtheta(x) = W x + b, \qquad W \in \mathbb{R}^{g \times d},\; b \in \mathbb{R}^g.
\end{equation}
These weights parameterize the QGK unitary $\hat{\mU}_{\vphi} = \hat{\mU}_{Wx + b}$ according to Eq.~\eqref{eq:unitary}. The embedded quantum state becomes $\psi(x) = \hat{\mU}_{Wx + b} \ket{\Psi}$.
The induced kernel from this quantum embedding is defined via fidelity:
\begin{equation}
    \mathcal{K}_{\phi}(x,x') =\mathcal{K}_{W,b}(x,x') = |\langle \psi(x'), \psi(x) \rangle|^2
    = \left| \bra{\Psi} \hat{\mU}^\dagger_{Wx' + b} \hat{\mU}_{Wx + b}\ket{\Psi}\right|^2.
\end{equation}

\paragraph{Role of the Affine Projection}
The affine transformation $\vphi = W x + b$ plays two essential roles:
\begin{itemize}[left=4pt]
    \item \textbf{Controlled generator activation:}  
    It determines how the input modulates generator weights $\{ \hat{\mH}_i \}$, shaping the embedding's structure in Hilbert space.

    \item \textbf{Task-dependent kernel shaping:}  
    During pre-training, $(W, b)$ are optimized using the Kernel Target Alignment objective (\autoref{eq:KTA}) to align the induced kernel with the downstream task.
\end{itemize}

Despite this learned classical transformation, no classical nonlinearities are introduced: all nonlinearity arises solely from the quantum unitary exponentiation and subsequent fidelity computation.

\paragraph{RKHS-Constrained Expressivity}
Once the projection $(W, b)$ is fixed, QGK predictions take the standard kernel-learning form and hence lie in the RKHS $\mathcal{H}_{\mathcal{K}_{W,b}}$ associated with the induced kernel.
The projection determines \emph{which} RKHS is selected but does not extend the model class beyond it.

\begin{proposition}[RKHS-Constrained Expressivity]
Let the quantum feature map be defined as $\psi(x) = \hat{\mU}_{W x + b} \ket{\Psi}$, inducing the kernel $\mathcal{K}_{W,b}(x,x') = |\langle \psi(x'), \psi(x) \rangle|^2$.
Then, for fixed $(W, b)$, any classifier of the form $f(x) = \sum_{i=1}^n \alpha_i\, \mathcal{K}_{W,b}(x,x_i)$ belongs to the RKHS $\mathcal{H}_{\mathcal{K}_{W,b}}$.
Thus, while $(W, b)$ select the kernel, they do not expand the hypothesis class beyond that space.
\end{proposition}

\paragraph{Comparison with Hybrid Quantum Models}
Architectures like the \textit{Dressed Quantum Circuit} (DQC) \citep{Mari2020transferlearningin} apply classical nonlinear networks both before and after the quantum embedding. Formally, they compute:
\begin{equation}
f_\theta(x) = g_{\text{post}} \big( \psi_\theta( g_{\text{pre}}(x) ) \big),
\end{equation}
where both $g_{\text{pre}}$ and $g_{\text{post}}$ are deep neural networks with nonlinear activations (e.g., $\varphi(Wx + b)$).
By the universal approximation theorem \citep{hornik1989multilayer}, a sufficiently wide or deep classical network can approximate any continuous function on a compact domain. Consequently, the classical component $g_{\text{post}} \circ g_{\text{pre}}$ is already expressive enough to model arbitrary functions --- even if the quantum circuit contributes no transformation at all (e.g., implements the identity).
Consequently, DQC expressivity may be dominated by classical nonlinearities, making it difficult to isolate or attribute performance to quantum processing.
In contrast, our QGK architecture constrains expressivity strictly to the RKHS induced by the quantum kernel, enabling clearer theoretical analysis and benchmarking of the quantum contribution: 
\begin{itemize}[left=4pt]
    \item QGK uses uses only an \emph{affine} classical preprocessing layer ($Wx + b$),
    \item no post-processing nonlinearities are applied after the quantum embedding,
    \item all nontrivial expressivity arises solely from the quantum feature map $\psi(x)$.
\end{itemize}

\paragraph{Reweighted Group Contributions}
We now extend the expressivity bounds derived in Appendix~\ref{app:expressivity_bounds} to incorporate the affine preprocessing layer.
Although the grouping structure $\{|\mathcal{G}_i|\}$ remains unchanged, the affine map modifies \emph{how strongly} each group contributes on average to the
parameterized Hamiltonian.  
If $x$ is drawn from a distribution with covariance
$\Sigma_x \preceq \sigma_x^2 I$, then the variance of the affine coefficient satisfies $\mathrm{Var}[\phi_i(x)]=W_{i,:}\,\Sigma_x\,W_{i,:}^\top\le\sigma_x^2 \|W_{i,:}\|_2^2$.
Since the grouped generators obey $\|\hat{\mH}_i\|_F^2=|\mathcal{G}_i|$,
the variance of the random Hermitian term is bounded by
\begin{equation}
    \mathrm{Var}\!\big[\phi_i(x)\hat{\mH}_i\big]
    \;\le\;
    \sigma_x^2\,|\mathcal{G}_i|\,\|W_{i,:}\|_2^2.
    \label{eq:variance-reweight}
\end{equation}
Thus, the affine map preserves the structural group sizes $|\mathcal{G}_i|$, but it \emph{rescales their influence} in the Hamiltonian by the factor $\|W_{i,:}\|_2^2$. 
This motivates replacing the terms $|\mathcal{G}_i|$ in Theorem~\ref{thm:kl_bounds}  with their reweighted form, yielding the same group-size bounds as before, but with each group’s structural size weighted by the average magnitude with which the affine layer excites that group:
\begin{equation}
    \frac{c_1'}{n}\,
    \frac{\sum_i |\mathcal{G}_i|\,\|W_{i,:}\|_2^2}
         {(\sum_i \sqrt{|\mathcal{G}_i|}\,\|W_{i,:}\|_2)^2}
    \;\le\;
    \mathcal{E}(K_{W,b})
    \;\le\;
    c_2'\,
    \frac{\sum_i |\mathcal{G}_i|^2\,\|W_{i,:}\|_2^4}
         {\sum_i |\mathcal{G}_i|\,\|W_{i,:}\|_2^2}
\label{eq:reweighted-bounds}
\end{equation}

The upper bound (RHS in Eq.~\ref{eq:reweighted-bounds}) is governed by the \emph{anisotropy ratio}, which quantifies how unevenly the generator mass is distributed across groups after preprocessing.
Since the grouped generators form an orthogonal partition of the elementary basis, the relevant contributions depend only on the per-group magnitudes induced by $W$, rather than on cross-terms between different groups. 
Consequently, the upper bound is tightened precisely when preprocessing equalizes these per-group magnitudes -- i.e., when $\|W_{i,:}\|_2^2$ is approximately constant across i (up to the structural group sizes $|\mathcal G_i|$). 
Under KTA optimization (and/or explicit norm-regularization), $W$ is encouraged to balance excitation across groups, which reduces the anisotropy ratio and thus yields a strictly tighter KL upper bound compared to anisotropic configurations.

\paragraph{Conclusion}
The classical affine projection in the QGK provides a compact, trainable interface to parameterize the generator weights and to align the quantum kernel with the target task via KTA pre-training.  
Because predictions depend exclusively on the kernel $\mathcal{K}_{W,b}$, model expressivity remains strictly bounded by the RKHS associated with that kernel.  
This avoids the uncontrolled representational capacity introduced by classical deep post-processing and preserves a clean theoretical separation between classical preprocessing and quantum feature generation.
Importantly, the affine map effectively reweights generator directions, reducing anisotropy and tightening the upper expressivity bound. 
These benefits arise already for random weights and are further strengthened by KTA pre-training.
Thus the linear projection reshapes the geometry of the kernel within its RKHS, selecting a better kernel without adding any classical nonlinear expressivity.

\clearpage

\section{Quantum Generator Kernel Computational Complexity}\label{app:complexity}

Despite the exponential scaling in $\eta$, the QGK remains classically simulable via efficient tensor operations. 
To assess the practical efficiency of this approach, we derive the full classical sample complexity of the Quantum Generator Kernel (QGK) and contrast it with classical kernels to determine scalability conditions and break-even points.

\begin{axiom}[QGK Execution Complexity]
Given $n$ samples of $d$-dimensional inputs, $\eta$ qubits, and $g$ variational generator groups (VGGs), the cost of executing the QGK kernel is decomposed as:
\begin{itemize}[topsep=-2pt,left=4pt]
\item \textbf{Generator construction: $\mathcal{O}(4^\eta)$.} We precompute the set of $4^\eta-1$ generators spanning the $2^\eta$-dimensional Hilbert space (cf. \autoref{alg:generator_construction} and group them (cf. \autoref{alg:vgg_construction}); this is a one-time cost independent of $n$.
  
\item \textbf{Input projection $\vphi: \mathbb{R}^d \to \mathbb{R}^g$: $\mathcal{O}(n\cdot g\cdot d)=\mathcal{O}(n\cdot \gamma g^2)$.} 
For $n$ inputs, we compute $g$ projected features, each as a weighted sum over $d$ input dimensions.  
Using $\gamma=d/g$ (i.e., $d=\gamma g$), this becomes $O(n\gamma g^2)$.

\item \textbf{Hamiltonian embedding: $\mathcal{O}(n\cdot 8^\eta)$.} 
For $n$ inputs, we form a dense $2^\eta\times 2^\eta$ Hamiltonian matrix $H(x)=\sum_{i=1}^g \phi_i(x)\,\hat H_i$ and compute a \emph{single} matrix exponential $U(x)=\exp(-iH(x))$.  
The dominant cost is the matrix exponential on dimension $2^\eta$, which scales cubically as $O((2^\eta)^3)=O(8^\eta)$ per sample. 

\item \textbf{Kernel matrix computation: $\mathcal{O}(n^2\cdot 2^\eta)$.} 
We simulate \emph{statevectors} and compute the kernel via batched fidelities:
using $\varphi(x)=U(x)\ket{0}\in\mathbb{C}^{2^\eta}$, we form the Gram matrix of overlaps $G_{ij}=\langle \varphi(x_i),\varphi(x_j)\rangle$ using a batched matrix multiply between an $n\times 2^\eta$ state matrix and its adjoint, which costs $O(n^2\cdot 2^\eta)$.
The kernel is then obtained as $K_{ij}=|G_{ij}|^2$ with only elementwise post-processing.
Notably, this stage does \emph{not} require multiplying $2^\eta\times 2^\eta$ unitaries per pair.
\end{itemize}
\end{axiom}

By variablizing the compression ratio $\gamma=\frac{d}{g}$, or, $d=\gamma g$, we can represent the default case of a 1:1 group-to-feature mapping using $\gamma=1$.

\begin{lemma}\label{lemma:compplexity}
Overall, we can summarize the classical computational complexity of the QGK as:
\begin{align}
\mathcal{C}_{QGK}=\mathcal{O}(4^\eta + n \cdot \gamma \cdot g^2 + n \cdot 8^\eta + n^2 \cdot 2^\eta)
\end{align}
\end{lemma}

For comparison, a classical RBF or Linear kernel computes a similarity matrix at $\mathcal{O}(n^2 \cdot d)$.
\autoref{tab:benchmark-complexity} shows a comparison over the end-to-end complexities of the QGK and Classical Kernels for the utilized benchmarks using a 90/10 train/test split, where QGK is more efficient throughout: 
\begin{table}[ht]\centering\begin{tabular}{|l|c|c|}\hline
\textbf{Component} & \textbf{QGK (ours)} & \textbf{Classical Kernel} \\ \hline
moons ($\eta=2, d=2, n=200, g = 15$) & $\mathcal{O}(1.50e+05)$ & $\mathcal{O}(6.56e+04)$ \\
circles ($\eta=2, d=2, n=200, g = 15$) & $\mathcal{O}(1.50e+05)$ & $\mathcal{O}(6.56e+04)$ \\
Bank ($\eta=2, d=16, n=200, g = 15$) & $\mathcal{O}(1.92e+05)$ & $\mathcal{O}(5.25e+05)$ \\
MNIST ($\eta=5, d=784, n=1000, g = 93$) & $\mathcal{O}(1.32e+08)$ & $\mathcal{O}(6.43e+08)$ \\
CIFAR10 ($\eta=5, d=3072, n=1000, g = 93$) & $\mathcal{O}(3.45e+08)$ & $\mathcal{O}(2.52e+09)$ \\ \hline
\end{tabular}
\caption{End-to-end complexity comparison between QGK and classical kernels.}\label{tab:benchmark-complexity}
\end{table}

To generally determine when the QGK is more efficient than classical kernels, we consider:
\begin{equation}
4^\eta + n \cdot \gamma \cdot g^2 + n \cdot 8^\eta + n^2 \cdot 2^\eta < n^2 \cdot d
\end{equation}

which, using $d=\gamma g$ can be simplified to the quadratic inequality in $n$:
\begin{equation}
n^2(2^\eta - \gamma g) + n (\gamma g^2 + 8^\eta) + 4^\eta < 0
\end{equation}
which, assuming $A<0$ and $B^2 - 4AC>0$ (we outline $B^2 \gg 4AC>0$ below), can be solved to:
\begin{equation}\label{eq:dataset_threshold}
n > \frac{-B - \sqrt{B^2 - 4AC}}{2A} 
\quad \text{where} \quad
\begin{cases}
A = 2^\eta - \gamma g \\
B = \gamma g^2 + 8^\eta \\
C = 4^\eta
\end{cases}
\end{equation}
Note that we use $-B-\sqrt{\Delta}$, which, due to $A<0$, becomes overall more positive than $-B+\sqrt{\Delta}$.
Furthermore, a necessary condition for $A<0$ is $2^\eta < \gamma g$, which, using the default grouping (cf., \autoref{eq:groups}) and the default group-to-feature mapping $\gamma=1$ yields $2\cdot2^\eta>-3$ for even $\eta$ and $2\cdot2^\eta>3$ if $\eta$ is odd, which trivially holds for $\eta>0$.
Summarizing the above derivation, we can determine the lower QGK efficiency bound: $\epsilon b_\gamma = \frac{n}{d}$, where larger intended ratios imply $C_{GQK} < C_{Classical}$ and verce visa.
For the default 1:1 group-to-feature mapping ($\gamma=1$) we therefore get:

\begin{table}[h!]\centering\begin{tabular}{|c|c|c|c|c|c|c|c|}
\hline
$\eta$ & 2 & 3 & 4 & 5 & 6 & 7 & 8 \\ \hline
$\epsilon b_1$ & \textbf{1.76} &\textbf{ 3.49} & \textbf{3.75} & \textbf{7.30} & 11.75 & 23.26 & 43.75 \\ \hline
\end{tabular}\end{table}
Using the general assumption $n \gg d$, i.e., favoring high ratios, the above table shows that for low qubit counts classical QGK execution excels classical kernels. For example assuming a minimum dataset-to-input-dimension rate of eight ($\epsilon b_\gamma>8$), this holds for $\eta\leq5$ with $\gamma=1$.

To further facilitate higher qubit counts and larger input dimensionalities, we suggest adapting the input compression $\gamma$ accordingly. E.g., for $\gamma=\eta$ we get the following approximately constant ratio $\epsilon b_\eta < 1$, which, under the assumption that $n \gg d$, demonstrates that under reasonable compression, executing the QKG classical is computationally more efficient than classical kernels throughout varying numbers of qubits.

\begin{table}[ht]\centering\begin{tabular}{|c|c|c|c|c|c|c|c|}\hline
$\eta$ & 2 & 3 & 4 & 5 & 6 & 7 & 8 \\ \hline
$\epsilon b_\eta$ & \textbf{0.66} & \textbf{0.53} &\textbf{0.38} & \textbf{0.38} & \textbf{0.38} & \textbf{0.46} & \textbf{0.59} \\ \hline
\end{tabular}\end{table}

These findings are summarized in \autoref{fig:scale:complexity} where the dataset size threshold ($n$ according to \autoref{eq:dataset_threshold}) is plotted for $\gamma=1$ ($\circ$-markers) and $\gamma=2^\eta$ ($\scriptstyle{\square}$-markers) on the left y-axis. The area above the resulting graphs indicates that classically executing the QGK is more efficient than a classical kernel like RBF or Linear. 
Combined with the number of generators, dictating the available input dimensions on the right y-axis, scaled to reflect $n\gg d$ ($\gamma g=d=n /10$), the breakeven point regarding the classical complexity can be observed at $\eta\leq 5$ for $\gamma=1$.
Increasing the compression to $\gamma=\eta$ pushes the lower efficiency bound further down, such that the input dimension reference ($d=\eta g$ for $g$ VGGs, light blue) no longer intersects. 
This demonstrates that such a hybrid approach is an efficient complexity mitigation to enable scalable QGK execution in the current NISQ era, where especially larger quantum systems cannot be efficiently simulated and available quantum hardware is prone to hardware noise.
We exemplify this approach with the MNIST and CIFAR10 benchmarks, where, using $\eta=5$, we have $g=93$ VGGs, resulting in compression rates of $\gamma\approx8$ and $\gamma\approx32$ respectively. 
Notably, this compression results in a number of groups that show to satisfy the intended ratio of $n=10d$ between (kernel)-input and dataset dimension, with $d=1000$.

Overall, we can generalize the above observations in the following two theorems regarding the classical computational complexity of the QGK compared to classical kernels. 
First, we derive the following simplifications and approximations from \autoref{eq:dataset_threshold} based on \autoref{lemma:compplexity}:

\begin{lemma}[Simplify complexity bounds]

Given $\gamma\ge1$, we can show $B^2 \gg 4AC$, or, using $g\approx 3 \cdot 2^\eta$, $(9\gamma \cdot 4^\eta + 8^\eta)^2 \gg 8^\eta (4 - 12 \gamma)$, which trivially holds for all $\eta\ge1$, where, the LHS is largely positive and increasing, while the RHS is negative and decreasing with increasing $\eta$. Therefore, we can further approximate:
\begin{equation}
\frac{-B - \sqrt{B^2 - 4AC}}{2A} \approx \frac{-2B}{2A} = \frac{B}{-A} = \frac{\gamma g^2 + 8^\eta}{\gamma g - 2^\eta}
\end{equation}

With $d=\gamma g$ and given the number of groups is chosen according to \autoref{eq:groups}, we can further substitute $g = 3 \cdot 2^\eta - 6 \cdot (\eta \mod 2 ) + 3 \approx 3 \cdot 2^\eta$, yielding:
\begin{equation}
n>\frac{\gamma g^2 + 8^\eta}{\gamma g - 2^\eta} \approx \frac{9\gamma \cdot 4^\eta + 8^\eta}{3\gamma \cdot 2^\eta - 2^\eta}
\end{equation}

Thus, $\epsilon b_\gamma$ can be further approximated: 
\begin{equation}\label{eq:eb_simple}
\epsilon b_\gamma  \frac{n}{d} \approx \frac{\frac{9\gamma \cdot 4^\eta + 8^\eta}{3\gamma \cdot 2^\eta - 2^\eta}}{3\gamma \cdot 2^\eta} = \frac{9\gamma \cdot 4^\eta + 8^\eta}{3\gamma(3\gamma - 1) \cdot 4^\eta}
\end{equation}
\end{lemma}

These simplifications help us to clearly demonstrate the following theorems:

\begin{theorem}
Classically computing the QGK is more efficient than using a classical kernel such as Linear or RBF, for small qubit numbers of $\eta\le5$, when using no compression. 
\end{theorem}

\begin{proposition}
For fitting kernel classifiers, $n\gg d$ is a generally assumed condition. To determine the efficiency bound, we specifically assume $n=10d$. 
\end{proposition}
\begin{proof}
Using no compression, i.e.,,$\gamma=1$, we can derive the following special case from \autoref{eq:eb_simple}:
\begin{equation}
\epsilon b_1 = \frac{9 \cdot 4^\eta + 8^\eta}{6 \cdot 4^\eta} = \frac{3}{2} + \frac{2^\eta}{6}
\end{equation}
With $\epsilon b_1 < 10$, we get $2^\eta < 51$, or, $\eta \lessapprox 5.67$.
Thus, $n=10d$ holds for $\eta\leq 5$ with $\gamma=1$
\end{proof}

\begin{theorem}
Using sufficient compression, QGKs can be classically executed in a hybrid fashion that is more computationally efficient than classical kernels. 
\end{theorem}

\begin{proof}
We aim to demonstrate finding a compression bound, s.t. $\epsilon b_\gamma < 1$ with $\gamma>1$.
Based on \autoref{eq:eb_simple}, we rewrite the above condition to: 
\begin{align}
\epsilon b_\gamma < 1 \iff 1 &> \frac{9\gamma \cdot 4^\eta + 8^\eta}{3\gamma(3\gamma - 1) \cdot 4^\eta} 
\iff
0 < 12\gamma - 9\gamma^2 + 2^\eta \\
\gamma &> \frac{12 + \sqrt{144 + 36 \cdot 2^\eta}}{18}
= \frac{2\sqrt{4+2^\eta}}{3}
\end{align}
Thus, with sufficient compression, $\gamma > \frac{2\sqrt{4+2^\eta}}{3}$, we can ensure $\epsilon b_\gamma<1$, thus, a more efficient classical execution of the QGK in datasets with $n>d$.
\end{proof}

\clearpage

\section{Additional Results}\label{app:analysis}

\paragraph{Hardware Considerations}
Given current and near-term quantum devices, we distinguish three application scenarios for QGK execution. 
For small-scale datasets and low-depth settings ($\eta < 5$), QGK circuits can be realistically deployed on existing quantum devices. 
Our simulations confirm noise robustness up to $\sim 100$ compiled gates and comptetitive performance under noise for $\eta=5$, making the QGK viable for near-term experimental evaluations. 
For medium-scale tasks, hybrid execution offers a practical near-term strategy: 
classical preprocessing can reduce input dimensionality before quantum embedding, enabling robust kernel computation on noisy devices without excessive circuit depth.  
For large-scale datasets such as \texttt{MNIST} or \texttt{CIFAR-10}, efficient tensor-based implementations with compression provide a tractable classical alternative, ensuring generator-based kernels remain competitive until fault-tolerant systems are available.  
In the long term, on future fault-tolerant quantum systems, QGK could be implemented end-to-end, including training of linear projection weights via variational optimization, enabling scalable, expressive, and fully quantum kernel learning on large-scale datasets. 
Furthermore, the QGK admits natural extensions that exploit additional qubits to reduce circuit depth rather than to increase parameterization. In particular, grouped Hamiltonian structure enables block-encoding and group-wise parallelization, where subsets of generators are executed on disjoint registers or via ancilla-assisted dilations. 

\begin{table}[ht]\centering\small
\begin{tabular}{|c|c|c|c|c|c|c|c|c|}\hline
Dataset  & $d$ & \# classes & Kernel & $\eta$ & \shortstack{Input \\ Features} & \shortstack{Reuploading \\ Layers} & \shortstack{Compiled \\ Depth L2} & \shortstack{Compiled \\ Depth L1} \\ \hline
\multirow{5}{*}{\shortstack{\texttt{moons}\\\texttt{circles}}} 
  & \multirow{5}{*}{$2$} & \multirow{5}{*}{$2$} 
  & QGK (ours) & $2$ & $15$ & -    & $17$ & $28$ \\ \cline{4-9}
&&& QEK        & $2$ &  $2$ & $1$  & $16$ & $50$ \\ 
&&& QEK-N      & $2$ &  $2$ & $0$  &  $6$ & $26$ \\ \cline{4-9}
&&& HEE        & $2$ &  $2$ & $1$  & $10$ & $18$ \\ 
&&& HEE-D      & $2$ &  $2$ & $2$  & $10$ & $30$ \\ \cline{4-9}
&&& PQK        & $3$ &  $2$ & $10$ & $362$ & $362$\\ \hline

\multirow{6}{*}{\texttt{bank}} & \multirow{6}{*}{$16$} & \multirow{6}{*}{$2$} 
  & QGK (ours) & $2$ & $15$ & - & $17$ & $28$\\ 
&&& QGK Static & $3$ & $16$ & - & $212$ & $248$\\ \cline{4-9}
&&& QEK        & $2$ & $16$ & $1$ & $17$ & $380$\\ 
&&& QEK-N      & $2$ & $16$ & $0$ & $17$ & $191$\\ \cline{4-9}
&&& HEE        & $2$ &  $2$ & $1$ & $10$ & $18$ \\ 
&&& HEE-D      & $2$ &  $2$ & $2$  & $10$ & $30$ \\ \cline{4-9}
&&& PQK        & $3$ &  $2$ & $10$ & $362$ & $362$\\ \hline

\multirow{5}{*}{\texttt{MNIST}} & \multirow{6}{*}{$784$} & \multirow{6}{*}{$10$} 
  & QGK (ours) & $5$ & $93$ & - & $4455$ & $4754$ \\ \cline{4-9}
&&& QEK        & $5$ & $784$ & $1$  & $22709$ & $24084$ \\ 
&&& QEK-N      & $5$ & $784$ & $0$  & $11355$ & $12006$ \\ \cline{4-9}
&&& HEE        & $5$ & $5$ & $1$  & $53$ & $53$ \\
&&& HEE-D      & $5$ & $5$ & $165$  & $4481$ & $4481$ \\ \cline{4-9}
&&& PQK        & $6$ &  $5$ & $10$ & $416$ & $416$ \\ \hline

\multirow{6}{*}{\texttt{CIFAR10}} & \multirow{6}{*}{$3072$} & \multirow{6}{*}{$10$} 
  & QGK (ours) & $5$ & $93$ & - & $4455$ & $4754$ \\ 
&&& QGK Static & $6$ & $3072$ & - & $18823$ & $19644$\\ \cline{4-9}
&&& QEK        & $5$ & $3072$ & $1$  & $88963$ & $94485$ \\ 
&&& QEK-N      & $5$ & $3072$ & $0$  & $44461$ & $47223$\\ \cline{4-9} 
&&& HEE        & $5$ & $5$ & $1$  & $53$ & $53$  \\
&&& HEE-D      & $5$ & $5$ & $165$  & $4481$ & $4481$  \\ \cline{4-9}
&&& PQK        & $6$ &  $5$ & $10$ & $416$ & $416$ \\ \hline

\end{tabular}
\caption{Parameter Overview: Adapted parameters to ensure comparable prerequisites and capabilities via similar numbers of qubits and compiled circuit depths with level 1 and 2 optimization. Introducing two additional ablations; HEE-D with additional layers, QEK-N without reuploading.}\label{tab:hyperparameters}
\end{table}

A detailed comparison of compiled depths and input dimensionality across QGK and baseline approaches is provided in \autoref{tab:hyperparameters}, illustrating their markedly different scaling behaviors: 
HEE and PQK maintain shallow and medium depths even for higher qubit counts but cannot scale to high-dimensional inputs (since the number of inputs scales linear with the number of qubits), thus heavily relying on classical preprocessing such as PCA or linear projections (HEE-Linear). 
QEK, by contrast, adapts to high-dimensional inputs with low qubit counts through reuploading, but this leads to impractical compiled depths as dimensionality grows. 
QGK offers a compromise, supporting a high number of input features even with few qubits, with compiled depth scaling more favorably while reducing reliance on classical preprocessing.

Finally, \autoref{tab:final_acc} and \autoref{tab:simulation_acc} report the final test accuracies of all evaluated approaches, under classical execution and simulated hardware noise models, respectively.
Approaches with extensive compiled depths are omitted from simulation.

\begin{table}[ht]\centering
\def\arraystretch{1.25}
\resizebox{\textwidth}{!}{%
\begin{tabular}{|c|c|c|c|c|c|c|}\hline
Method
& \shortstack{\texttt{moons} 2-class \\ ($d{=}2,n{=}200$)}
& \shortstack{\texttt{circles} 2-class \\ $(d{=}2,n{=}200$)}
& \shortstack{\texttt{bank} 2-class \\ ($d{=}16,n{=}2k$)}
& \shortstack{\texttt{MNIST} 10-class\\ ($d{=}784,n{=}10k$)}
& \shortstack{\texttt{CIFAR} 10-class\\ ($d{=}3072,n{=}10k$)} \\ \hline

QGK (ours) & $\mathbf{0.96 \pm 0.04}$ & $\mathbf{0.68 \pm 0.05}$ & $\mathbf{0.85 \pm 0.01}$ & $\mathbf{0.94 \pm 0.01}$ & $\mathbf{0.41 \pm 0.02}$ \\ 

QEK & $0.91 \pm 0.05$ & $0.58 \pm 0.06$ & $0.69 \pm 0.03$ & $0.11 \pm 0.01$ & $0.10 \pm 0.01$ \\ 

QEK-N & $0.86 \pm 0.05$ & $0.62 \pm 0.08$ & $0.66 \pm 0.03$ & $0.17 \pm 0.01$ & $0.10 \pm 0.01 $  \\ 

HEE Linear & $0.86 \pm 0.05$ & $0.55 \pm 0.08$ & $0.70 \pm 0.01$ & $0.71 \pm 0.02$ & $0.34 \pm 0.01$ \\

Linear KTA & $0.86 \pm 0.05$ & $0.43 \pm 0.11$ & $0.79 \pm 0.03$ & $0.92 \pm 0.01$ & $0.38 \pm 0.01$ \\

MLP & $0.81 \pm 0.08$ & $0.46 \pm 0.08$ & $0.66 \pm 0.04$ & $0.90 \pm 0.01$ & $0.35 \pm 0.02$ \\ 

\hline

QGK Static & $0.93 \pm 0.04$ & $0.59 \pm 0.07$ & $0.75 \pm 0.03$ & $0.82 \pm 0.01$ & -- \\ 

HEE & $0.89 \pm 0.05$ & $0.65 \pm 0.07$ & $0.61 \pm 0.03$ & $0.48 \pm 0.00$ & $0.20 \pm 0.01$ \\ 

HEE-D & $0.93 \pm 0.03$ & $0.58 \pm 0.08$ & $0.61 \pm 0.03$ & $0.43 \pm 0.01$ & $0.21 \pm 0.01$ \\ 

PQK & $0.71 \pm 0.13$ & $0.48 \pm 0.08$ & $0.56 \pm 0.05$ & $0.31 \pm 0.02$ & $0.21 \pm 0.01$ \\

RBF & $0.93 \pm 0.04$ & $0.64 \pm 0.11$ & $0.76 \pm 0.02$ & $0.92 \pm 0.01$ & $0.39 \pm 0.01$ \\ 

Linear & $0.86 \pm 0.05$ & $0.43 \pm 0.10$ & $0.79 \pm 0.02$ & $0.92 \pm 0.01$ & $0.33 \pm 0.01$ \\ \hline
\end{tabular}}

\caption{Final test accuracies for all methods across five benchmarks (\autoref{fig:training}). The best result per dataset is highlighted in bold. KTA-trained approaches are shown above the midrule, static approaches below. The QGK achieves top performance in all benchmarks.}\label{tab:final_acc}
\end{table}

\begin{table}[ht]\centering\scriptsize
\def\arraystretch{1.25}
\resizebox{\textwidth}{!}{%
\begin{tabular}{|c|c|c|c|c|c|}\hline
Method
& \shortstack{\texttt{moons} 2-class \\ ($d{=}2,n{=}200$)}
& \shortstack{\texttt{circles} 2-class \\ $(d{=}2,n{=}200$)}
& \shortstack{\texttt{bank} 2-class \\ ($d{=}16,n{=}200$)}
& \shortstack{\texttt{MNIST} 10-class\\ ($d{=}784,n{=}200$)}
& \shortstack{\texttt{CIFAR} 10-class\\ ($d{=}3072,n{=}200$)} \\ \hline

QGK (ours)
& $\bif{0.96 \pm 0.03}$
& $\bif{0.67 \pm 0.05}$
& $\bif{0.88 \pm 0.06}$
& $\bif{0.81 \pm 0.09}$
& $\bif{0.26 \pm 0.06}$ \\ 

QEK
& $\itf{0.90 \pm 0.06}$
& $\itf{0.53 \pm 0.09}$
& $\itf{0.72 \pm 0.13}$
& $\itf{0.04 \pm 0.05}$
& - \\ 

QEK-N
& $\itf{0.86 \pm 0.06}$
& $\itf{0.59 \pm 0.08}$
& $\itf{0.66 \pm 0.09}$
& $\itf{0.11 \pm 0.05}$
& $\itf{0.11 \pm 0.07}$ \\ 

\textcolor{gray}{MLP}
& \textcolor{gray}{$0.81 \pm 0.08$}
& \textcolor{gray}{$0.46 \pm 0.08$}
& \textcolor{gray}{$0.64 \pm 0.09$}
& \textcolor{gray}{$0.81 \pm 0.10$}
& \textcolor{gray}{$0.27 \pm 0.08$} \\\hline 

QGK Static
& $\itf{0.93 \pm 0.04}$
& $\itf{0.57 \pm 0.07}$
& $\itf{0.72 \pm 0.11}$
& $\itf{0.62 \pm 0.11}$
& $\itf{0.08 \pm 0.05}$ \\ 

HEE
& $\itf{0.88 \pm 0.04}$
& $\itf{0.61 \pm 0.07}$
& $\itf{0.59 \pm 0.11}$
& $\itf{0.21 \pm 0.08}$
& $\itf{0.17 \pm 0.07}$ \\
PQK
& $\itf{0.53 \pm 0.09}$
& $\itf{0.48 \pm 0.09}$
& $\itf{0.51 \pm 0.10}$
& $\itf{0.19 \pm 0.06}$
& $\itf{0.14 \pm 0.08}$ \\ 

\textcolor{gray}{RBF}
& \textcolor{gray}{$0.93 \pm 0.04$}
& \textcolor{gray}{$0.64 \pm 0.11$}
& \textcolor{gray}{$0.66 \pm 0.09$}
& \textcolor{gray}{$0.40 \pm 0.09$}
& \textcolor{gray}{$0.03 \pm 0.02$} \\
\textcolor{gray}{Linear}
& \textcolor{gray}{$0.86 \pm 0.05$}
& \textcolor{gray}{$0.43 \pm 0.10$}
& \textcolor{gray}{$0.71 \pm 0.09$}
& \textcolor{gray}{$\mathbf{0.86 \pm 0.08}$}
& \textcolor{gray}{$\mathbf{0.30 \pm 0.03}$} \\ \hline
\end{tabular}}
\caption{Noisy simulation results: Final mean classification accuracy and margin of error  obtained from 27-qubit IBM Falcon hardware \citep{qiskit2024} using $8192$ shots. The best result per dataset is highlighted in bold. A horizontal line separates pre-trained (KTA) approaches from static kernels. For reference, the final accuracies of all classical baselines in the reduced data-regime are reported in grey. Even under realistic noise, the QGK outperforms all quantum kernels.}\label{tab:simulation_acc}
\end{table}

\end{document}